\documentclass{article}
\pdfoutput=1

\PassOptionsToPackage{numbers, compress}{natbib}

\usepackage[preprint]{neurips_2023}

\usepackage{microtype}
\usepackage{graphicx}
\usepackage{subcaption}
\usepackage{booktabs} 
\usepackage{cancel}
\usepackage{wrapfig}
\usepackage{placeins}
\usepackage{lipsum}
\usepackage{yfonts}
\usepackage{bm}
\usepackage{blkarray}
\usepackage{etoolbox}
\usepackage{pifont} 
\usepackage[page,toc,titletoc,title]{appendix}

\usepackage{algorithm,lipsum}
\usepackage{algpseudocode}

\usepackage{hyperref} 
\hypersetup{colorlinks,linkcolor={blue},citecolor={blue},urlcolor={blue}}

\usepackage{nicefrac}       
\usepackage{amssymb,amsfonts}       
\usepackage{amsmath,stackengine,mathtools}
\usepackage{amsthm}
\usepackage{tcolorbox}
\usepackage{stmaryrd}
\usepackage[flushleft]{threeparttable}

\makeatletter
\def\moverlay{\mathpalette\mov@rlay}
\def\mov@rlay#1#2{\leavevmode\vtop{%
   \baselineskip\z@skip \lineskiplimit-\maxdimen
   \ialign{\hfil$\m@th#1##$\hfil\cr#2\crcr}}}
\newcommand{\charfusion}[3][\mathord]{
    #1{\ifx#1\mathop\vphantom{#2}\fi
        \mathpalette\mov@rlay{#2\cr#3}
      }
    \ifx#1\mathop\expandafter\displaylimits\fi}
\makeatother

\usepackage[utf8]{inputenc}

\usepackage{xspace}
\usepackage[T1]{fontenc}    
\usepackage[capitalise]{cleveref}
\usepackage{diagbox}
\crefformat{section}{\S#2#1#3}
\crefformat{subsection}{\S#2#1#3}
\crefformat{subsubsection}{\S#2#1#3}

\makeatletter

\newenvironment{shiftedflalign*}{
	\start@align\tw@\st@rredtrue\m@ne
	\qquad\qquad
}{
	\endalign
}
\makeatother

\theoremstyle{definition}
\newtheorem{definition}{Definition}

\theoremstyle{definition}
\newtheorem{theorem}{Theorem}

\theoremstyle{definition}

\theoremstyle{definition}

\theoremstyle{definition}

\theoremstyle{definition}

\theoremstyle{definition}
\newtheorem{corollary}{Corollary}

\crefname{axiom}{ax.}{axs.}
\Crefname{axiom}{Axiom}{Axioms}
\creflabelformat{axiom}{#2#1#3}
\crefname{definition}{Def.}{Defs.}
\Crefname{definition}{Definition}{Definitions}
\creflabelformat{definition}{#2#1#3}
\crefname{proposition}{prop.}{props.}
\Crefname{proposition}{Proposition}{Propositions}
\creflabelformat{proposition}{#2#1#3}
\crefname{remark}{remark}{remarks}
\Crefname{remark}{Remark}{Remarks}
\creflabelformat{remark}{#2#1#3}
\crefname{corollary}{corol.}{corol.}
\Crefname{corollary}{Corollary}{Corollaries}
\creflabelformat{corollary}{#2#1#3}

\usepackage{todonotes}

\usepackage{enumitem}
\setlist[itemize]{leftmargin=*}
\usepackage{multirow}

\usepackage{dsfont}
\usepackage{csquotes}
\usepackage{pgfplots}
\usepackage{pgfplotstable}
\pgfplotsset{/pgfplots/error bars/error bar style={solid}}
\usepackage{filecontents}
\usepackage{tikz}
\usetikzlibrary{shapes,decorations,arrows,calc,arrows.meta,fit,positioning,trees,backgrounds}
\tikzset{
    -Latex,auto,node distance =1 cm and 1 cm,semithick,
    state/.style ={ellipse, draw, minimum width = 0.7 cm},
    point/.style = {circle, draw, inner sep=0.04cm,fill,node contents={}},
    bidirected/.style={Latex-Latex,dashed},
    el/.style = {inner sep=2pt, align=left, sloped}
}

\pgfdeclarelayer{background}
\pgfsetlayers{background,main}

\newcounter{daggerfootnote}

\DeclareMathOperator*{\minimise}{minimise}

\allowdisplaybreaks
\usepackage{mathrsfs}
\tikzset{
  mybackground51/.style={execute at end picture={
      \begin{scope}[on background layer]
        \draw[black, rounded corners=2ex, fill=black!2] (current bounding box.south west)
        rectangle (current bounding box.north east);
        \node[draw,fill=white,ellipse,anchor=west,inner sep=1pt,minimum width=4ex] at (current bounding box.north
        west){#1};
      \end{scope}
    }},
}

\DeclareMathOperator*{\argmax}{argmax} 
\DeclareMathOperator*{\argmin}{arg\,min}

\let\emptyset\varnothing

\newcommand{\mat}[1]{\mathbf{#1}}
\renewcommand{\vec}[1]{\mathbf{#1}} 

\newcommand{\expectation}[2]{ \mathbb{E}_{#1}{\left[#2\right]} }
\newcommand{\X}{\vec{X}}
\newcommand{\N}{\vec{N}}
\newcommand{\Z}{\vec{Z}}
\newcommand{\Y}{\vec{Y}}

\newcommand{\D}{\vec{D}}
\newcommand{\E}{\vec{E}}
\newcommand{\V}{\vec{V}}
\newcommand{\U}{\vec{U}}

\newcommand{\x}{\vec{x}}

\newcommand{\T}{\mathcal{T}}

\newcommand{\observe}{\mathfrak{O}}
\newcommand{\intervene}{\mathfrak{I}}

\newcommand{\graph}{\mathcal{G}}
\newcommand{\scM}{\mathscr{M}}

\newcommand{\set}[1]{\{#1\}}

\newcommand{\DO}[2]{\operatorname{do} \!  \left(#1 = #2\right)}

\newcommand{\simpleDO}[1]{\text{do} \! \left(#1\right)}
\newcommand{\myP}[1]{P \! \left ( #1    \right)}
\newcommand{\myQ}[2]{Q^{#1}_{#2}}

\newcommand{\scmdef}{\left\langle \U,\V,\mat{F}, \myP{\U} \right\rangle}

\newcommand{\dom}[1]{\text{dom}(#1)}

\newcommand{\pa}[2]{\mathrm{pa}(#1)_{#2}}

\newcommand{\Pa}[2]{\mathrm{Pa}(#1)_{#2}}

\newcommand{\acro}[1]{\textsc{#1}\xspace}
\newcommand{\acros}[1]{\textsc{#1}s\xspace}

\newcommand{\mab}{\acro{mab}}
\newcommand{\mabs}{\acros{mab}}
\newcommand{\gp}{\acro{gp}}

\newcommand{\DAG}{\acro{dag}}
\newcommand{\pomis}{\acro{pomis}}
\newcommand{\mis}{\acro{mis}}
\newcommand{\uc}{\acro{uc}}
\newcommand{\mos}{\acro{mos}}

\newcommand{\bo}{\acro{bo}}
\newcommand{\cbo}{\acro{cbo}}
\newcommand{\dcbo}{\acro{dcbo}}
\newcommand{\ceo}{\acro{ceo}}
\newcommand{\mcbo}{\acro{mcbo}}

\newcommand{\mdp}{\acro{mdp}}

\newcommand{\scm}{\acro{scm}}

\newcommand{\rbf}{\acro{rbf}}
\newcommand{\rhs}{\acro{rhs}}
\newcommand{\osco}{\acro{osco}}

\title{Optimal Observation-Intervention Trade-Off \\
       in Optimisation Problems with Causal Structure
}
\author{%
  Kim Hammar\thanks{Equal contribution.}\\
  KTH Royal Institute of Technology\\
  Stockholm, Sweden \\
  \texttt{kimham@kth.se}
\setcounter{footnote}{1}
   \And
   Neil Dhir\footnotemark \\
   Siemens Technology\\
   Berkeley, California, USA \\
   \texttt{neil.dhir@siemens.com}
}

\begin{document}

\maketitle

\begin{abstract}
We consider the problem of optimising an expensive-to-evaluate grey-box objective function, within a finite budget, where known side-information exists in the form of the causal structure between the design variables. Standard black-box optimisation ignores the causal structure, often making it inefficient and expensive. The few existing methods that consider the causal structure are myopic and do not fully accommodate the observation-intervention trade-off that emerges when estimating causal effects. In this paper, we show that the observation-intervention trade-off can be formulated as a non-myopic optimal stopping problem which permits an efficient solution. We give theoretical results detailing the structure of the optimal stopping times and demonstrate the generality of our approach by showing that it can be integrated with existing causal Bayesian optimisation algorithms. Experimental results show that our formulation can enhance existing algorithms on real and synthetic benchmarks.
\end{abstract}

\section{Introduction}\label{sec:intro}
This paper studies global optimisation of an expensive-to-evaluate \emph{grey-box} \citep{astudillo2021thinking} objective function with known causal structure in the form of a \textit{causal diagram} (making it grey-box rather than `black-box'). In this setting, inputs to the objective function correspond to \textit{interventions} and outputs correspond to causal effects. We assume that the objective function can be evaluated (possibly with noise) at a finite number of inputs, either by measurement or some estimation procedure. Each evaluation is associated with a cost and a finite budget of total evaluations is prescribed. Since no known functional form of the objective function is available, our goal is to find an input that optimises the objective function by estimating the causal effects of a sequence of interventions. This estimation can be done in two ways: \textit{i}) by intervening and conducting controlled experiments; and \textit{ii}) by passively observing and using the causal graph to estimate the causal effects (vis-\`{a}-vis the do-calculus \citep{pearl2000causality}). In choosing between these two options, an \textit{observation-intervention} trade-off emerges. On the one hand, interventions are costly but allow us to reliably estimate causal effects. On the other hand, observations are (usually) cheap to collect but may not always be sufficient to identify causal effects. We show that this trade-off can be formulated as an \textit{optimal stopping} problem that permits an efficient solution \cite{wald,shirayev,chow1971great}.

Two principal algorithmic frameworks have been developed to solve optimisation problems of the type described above: \textit{i}) causal Bayesian optimisation (\cbo) algorithms \cite{cbo,dcbo,model_based_cbo,branchini23}, which assume that the objective function is defined over a continuous domain; and \textit{ii}) causal multi-armed bandit (\mab) algorithms \cite{causal_bandits_5,causal_bandits_4,causal_bandits_3,causal_bandits_2,causal_bandits_1}, which assume that the objective function is defined over a discrete set of inputs. To our knowledge, no hybrid approach exists. Compared to standard Bayesian optimisation (\bo) and \mabs, which ignore the causal structure of the problem, \cbo and causal \mabs are able to exploit the causal structure to improve sample efficiency. A drawback of the existing causal approaches, however, is that they are \textit{myopic} in the sense that they do not consider more than one step into the future when deciding on interventions. Another limitation is that the existing approaches rely on heuristics to balance the observation-intervention trade-off and do not quantify the cost of observing. Specifically, in the \cbo setting, an $\epsilon$-greedy strategy is adopted by \citet{cbo} and a myopic exploration approach is adopted by \citet{model_based_cbo} and the trade-off is ignored in \citep{dcbo,branchini23}. In the causal \mab setting, a heuristic approach is used by \citet{causal_bandits_5} and the work in \citep{causal_bandits_4,causal_bandits_3,causal_bandits_2,causal_bandits_1} uses myopic exploration approaches similar to \citet{model_based_cbo}. Moreover, all of the approaches referenced above assume a) that an arbitrary number of observations can be collected; and b) that there are no costs associated with collecting observations. These assumptions are not realistic in many scenarios: to test for high cholesterol in the US, a blood-test (observation) is required which costs on average \$51 \citep{cbc_test_cost} and the requisite intervention (statins) costs \$139 \citep{statin_cost}. Similarly, a prostate-specific antigen test to screen for prostate cancer (observation) costs on average \$40 \citep{blood_test} and the cost of a radical prostatectomy procedure (intervention) is on average \$34,720 \citep{pate2014variations,imber2020financial}.



\textbf{Main contributions.} Motivated by the above shortcomings in existing work, we present a general approach to extend existing \cbo  algorithms (the results are fully extendable to the causal \mab setting) to balance the intervention-observation trade-off in an optimal and \emph{non-myopic} way, while taking observation costs into account. From hereon, for brevity, we refer to our approach as \emph{Optimal Stopping for Causal Optimisation} (\osco). Our main contributions are:
\begin{itemize}
\item We formulate the observation-intervention trade-off as an \textit{optimal stopping} problem whose solution determines whether a given intervention should be carried out or if it is more cost-effective to collect observational data.
\item We prove that the solution to the optimal stopping problem can be computed efficiently and show that it can enhance existing \cbo (and causal \mab) algorithms.
\item We characterise a set of variables called the \emph{minimal observation set} (\mos), which is the minimal set of variables that need to be observed to estimate the causal effect of an intervention.
\end{itemize}

\section{Theoretical background}\label{sec:background}
This section covers notations and theoretical background on structural causal models. The models and definitions introduced here provide a foundation for the subsequent section where we describe our problem statement. A table of notations is available in \cref{sec:notations}. Finally, to ensure that the narrative remains fluid we also introduce the \mos in this section.

\subsection{Structural causal models}
\label{sec:scm}

\paragraph{Structural causal models.} A Structural Causal Model (\scm{}) \citep[Ch. 7, Def 7.1.1]{pearl2000causality} is a semantic framework to model the causal mechanisms of a system. Let $\mathscr{M}$ be a \scm parametrised by the quadruple $\scmdef$. Here $\U$ is a set of exogenous variables which follow a joint distribution $\myP{\U}$ and $\mat{V}$ is a set of endogenous (observed) variables. Within $\V$ we distinguish between three types of variables: manipulative $\X \subseteq \V \setminus \Y \setminus \N$; non-manipulative $\N \subseteq \V \setminus \X$ and targets (outcome variables) $\Y \subseteq \N$.

Graphically, each \scm induces a causal diagram (a directed acyclic graph, \DAG for short) $\graph  = \left\langle \mat{V}, \mat{E} \right\rangle$. Each vertex in the graph corresponds to a variable and the directed arcs point from members of $\pa{V_i}{\graph}$ and $U_i$ toward $V_i$, where $\pa{V_i}{\graph}$ denotes the parent nodes of $V_i$ in $\graph$ \citep[Ch. 7]{pearl2000causality}. The arcs represent a set of functions $\mat{F} \triangleq \{f_i\}_{V_i \in \mat{V}}$ \citep[\S 1]{lee2020characterizing}. Each function $f_i$ is a mapping from (the respective domains of $U_i \cup \pa{V_i}{\graph}$ to $V_i$ -- where $U_i \subseteq \mat{U}$ and $\pa{V_i}{\graph} \subseteq \mat{V} \setminus V_i$. A bidirected arc between $V_i$ and $V_j$ occurs if they share an unobserved confounder, i.e. if $\mat{U}_i \cap \mat{U}_j \neq \varnothing$ \citep{lee2018structural}. If each function $f_i \in \mat{F}$ is independent of time, the \scm{} is said to be stationary and if both $\mat{U}$ and $\mat{V}$ are finite the \scm{} is said to be finite. For a more incisive discussion on the properties of \scm{}s we refer the reader to \citep{pearl2000causality, bareinboim2016causal}.

\paragraph{Causal effects of interventions.} The do-operator $\DO{\X}{\x}$ represents the causal effect of an intervention that fixes a set of endogenous variable(s) $\X$ to constant value(s) $\x$ irrespective of their original mechanisms $\mat{F}$. This operation can be represented graphically by removing the incoming arcs to $\X$ in $\graph$. We denote by $\graph_{\overline{\X}}$ the mutilated graph obtained by deleting from $\graph$ all arcs pointing to nodes in $\X$. Examples of mutilated graphs are shown in \cref{fig:causal_graph_example}.

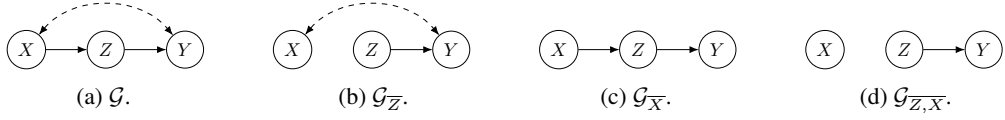
\begin{figure}[ht!]
    \centering
        \begin{subfigure}[t]{0.24\textwidth}
            \centering
            \resizebox{0.8\textwidth}{!}{%
               \begin{tikzpicture}[node distance =1.5cm]
                    \node[state,circle] (X0) {$X$};
                    \node[state,circle, right of = X0] (Z0) {$Z$};
                    \node[state,circle, right of = Z0] (Y0) {$Y$};
                    \path (X0) edge (Z0);
                    \path (Z0) edge (Y0);
                    \path[bidirected] (X0) edge[bend left=50] (Y0);
                \end{tikzpicture}
            }%
            \caption{$\graph$. \label{fig:admg_a}}
        \end{subfigure}
    \hfill
    \begin{subfigure}[t]{0.24\textwidth}
        \centering
        \resizebox{0.8\textwidth}{!}{%
            \begin{tikzpicture}[node distance =1.5cm]
                \node[state,circle] (X0) {$X$};
                \node[state,circle, right of = X0] (Z0) {$Z$};
                \node[state,circle, right of = Z0] (Y0) {$Y$};
                \path (Z0) edge (Y0);
                \path[bidirected] (X0) edge[bend left=50] (Y0);
            \end{tikzpicture}
        }%
        \caption{$\graph_{\overline{Z}}$.}
    \end{subfigure}
    \hfill
    \begin{subfigure}[t]{0.24\textwidth}
        \centering
        \resizebox{0.8\textwidth}{!}{%
           \begin{tikzpicture}[node distance =1.5cm]
                \node[state,circle] (X0) {$X$};
                \node[state,circle, right of = X0] (Z0) {$Z$};
                \node[state,circle, right of = Z0] (Y0) {$Y$};
                \path (X0) edge (Z0);
                \path (Z0) edge (Y0);
                \path[bidirected,opacity=0.0] (X0) edge[bend left=50] (Y0);
            \end{tikzpicture}
        }%
        \caption{$\graph_{\overline{X}}$. \label{fig:toy_graph}}
    \end{subfigure}
    \hfill
    \begin{subfigure}[t]{0.24\textwidth}
        \centering
        \resizebox{0.8\textwidth}{!}{%
           \begin{tikzpicture}[node distance =1.5cm]
                \node[state,circle] (X0) {$X$};
                \node[state,circle, right of = X0] (Z0) {$Z$};
                \node[state,circle, right of = Z0] (Y0) {$Y$};
                \path (Z0) edge (Y0);
                \path[bidirected,opacity=0.0] (X0) edge[bend left=50] (Y0);
            \end{tikzpicture}
        }%
        \caption{$\graph_{\overline{Z,X}}$.}
    \end{subfigure}
    \caption{Mutilated causal diagrams; nodes represent variables in an \scm; solid edges represent causal relations and dashed edges represent unobserved confounding; the diagrams represent post-intervention worlds where a specific intervention has been implemented, from left to right, these are: $\simpleDO{\emptyset},\simpleDO{Z},\simpleDO{X}$ and $\simpleDO{X,Z}$; interventions are graphically represented with the incoming arcs onto the intervention removed.}
    \label{fig:causal_graph_example}
\end{figure}


\paragraph{Estimating causal effects.} The goal of causal inference is to generate probabilistic formulas for the effects of interventions in terms of observation probabilities. In this work we accomplish this by employing \citet{pearl2000causality}'s do-calculus, which is an axiomatic system for replacing probability formulas containing the do-operator with ordinary conditional probabilities. Application of the do-calculus requires the interventions to be uniquely determined from $P(\V)$ and $\graph$. Determining if this is the case is known as the \emph{problem of identification} and has received considerable attention in the causal inference literature \citep{pearl1995probabilistic,pearl2000causality, shpitser2006identification, shpitser2008complete,tikka2019causal, tikka2018identifying}, formally:
\begin{definition}{Causal effect identifiability \citep[Def. 1]{bareinboim2012causal}.}
\label{def:ID}
Let $\X,\Y$ be two sets of disjoint variables and let $\graph$ be the causal diagram. The causal effect of an intervention $\DO{\X}{\x}$ on a set of variables $\Y$ is said to be identifiable from $P$ in $\graph$ if $\myP{\Y \mid \DO{\X}{\x}}$ is \emph{uniquely} computable from $P(\V)$ in any causal model that induces $\graph$.
\end{definition}

\subsection{Sets of endogenous variables}
\label{sec:vertex_sets}
\paragraph{Intervention sets.} Given a causal graph $\graph$ of an \scm $\mathscr{M} \triangleq \scmdef$ with a set of manipulative variables $\X$ and a target variable $Y$, we can define \emph{minimal intervention sets}, which represent non-redundant intervention sets for achieving an effect on $Y$:
\begin{definition}{Minimal intervention set (\mis) \citep[Def. 2]{causal_bandits_1}.}
\label{def:mis}
A set of manipulative variables $\X \subseteq \V \setminus \{Y\} \setminus \N$ is said to be a \mis for $Y$ if there is no $\X' \subset \X$ such that $\expectation{}{Y \mid \DO{\X}{\x}} = \expectation{}{Y \mid \DO{\X'}{\x'}}$ for every \scm induced by $\graph$. Denote by $\mathbf{M}^{\mathbf{V}}_{\graph,Y}$ a set of \mis{}s given the \scm $\scM$.
\end{definition}
A subset of the set of \mis{}s can lead to the minimal value of $Y$:
\begin{definition}{Possibly optimal minimal intervention set (\pomis) \citep[Def. 3]{causal_bandits_1}.}\label{def:pomis}
$\X \in \mathbf{M}_{\graph,Y}^{\mathbf{V}}$ is a \pomis if there exists an \scm induced by $\graph$ such that $\expectation{}{Y \mid \DO{\X}{\x^*}} < \expectation{}{Y \mid \DO{\vec{W}}{\vec{w}^*}} \ \forall \vec{W} \in \mathbf{M}_{\graph,Y}^{\mathbf{V}} \setminus \X$. Denote by $\mathbf{P}_{\graph,Y}^{\mathbf{V}}$ a set of \pomis{}s given the \scm $\scM$.
\end{definition}
\paragraph{Observation sets.} The solution to the \emph{identification problem} \citep{shpitser2008complete} tells us under what conditions the effect of a given intervention can be computed from $\myP{\V}$ and the causal diagram $\graph$. A number of sound and complete algorithms exist which solve this problem \citep{shpitser2006conditional, shpitser2008complete,tian2002general,huang2006identifiability}. The solution, if it exists as per \cref{def:ID}, returns an expression $\myQ{\X}{Y}$ which only contains observational measures. The set of variables $\Z \subseteq \V$ occurring in $\myQ{\X}{Y}$ is the \emph{minimal observation set}, which we introduce and formally define as
\begin{definition}{Minimal observation set (\mos).}\label{def:mos} If $\myP{Y \mid \DO{\X}{\x}}$ is identifiable as per \cref{def:ID} then $\exists \myQ{\X}{Y}$. If a) $\myQ{\X}{Y}$ can be estimated by observing $\Z \subseteq \V$ and b) $\nexists \Z' \subset \Z$ that allows us to estimate $\myQ{\X}{Y}$, then $\Z$ is a \mos. The \mos which follows $Q^{\X}_Y$ is denoted by $\mathbf{O}^{\mathbf{X}}_{\graph,Y}$.
\end{definition}
\begin{figure}
    \centering
    \begin{subfigure}[t]{0.31\textwidth}
        \centering
        \resizebox{0.9\textwidth}{!}{%
           \begin{tikzpicture}[node distance =1.5cm]
                \node[state,circle] (S) {$S$};
                \node[state,circle, below right of = S] (W) {$W$};
                \node[state,circle, below right of = W] (Y) {$Y$};
                \node[state,circle, above right of = Y] (X) {$X$};
                \node[state,circle, above right of = X] (Z) {$Z$};
                \node[state, circle, above right of = W] (B) {$B$};
                \path (S) edge (B);
                \path (B) edge (W);
                \path (W) edge (Y);
                \path (B) edge (X);
                \path (X) edge (Y);
                \path (Z) edge (X);
                \path[bidirected] (Y) edge[bend right=45] (Z);
                \path[bidirected] (S) edge[bend right=45] (Y);
            \end{tikzpicture}
        }%
        \caption{$\graph$. \label{fig:graph_two}}
    \end{subfigure}
    \hfill
    \begin{subfigure}[t]{0.31\textwidth}
        \centering
        \resizebox{0.9\textwidth}{!}{%
            \begin{tikzpicture}[node distance =1.5cm]
                \node[state,blue,line width=0.5mm,circle] (S) {$S$};
                \node[state,circle, below right of = S] (W) {$W$};
                \node[state,blue,line width=0.5mm,circle, below right of = W] (Y) {$Y$};
                \node[state,circle, above right of = Y] (X) {$X$};
                \node[state,circle, above right of = X] (Z) {$Z$};
                \node[state,blue,line width=0.5mm, circle, above right of = W] (B) {$B$};
                \path (B) edge (W);
                \path (W) edge (Y);
                \path (B) edge (X);
                \path (X) edge (Y);
                \path (Z) edge (X);
                \path[bidirected] (Y) edge[bend right=45] (Z);
                \path[bidirected] (S) edge[bend right=45] (Y);
            \end{tikzpicture}
        }%
        \caption{$\graph_{\overline{B}}$; see \labelcref{eq:do_X}.\label{fig:post_intervention1}}
    \end{subfigure}
    \hfill
    \begin{subfigure}[t]{0.31\textwidth}
        \centering
        \resizebox{0.9\textwidth}{!}{%
            \begin{tikzpicture}[node distance =1.5cm]
                \node[state,blue,line width=0.5mm,circle] (S) {$S$};
                \node[state,blue,line width=0.5mm,circle, below right of = S] (W) {$W$};
                \node[state,blue,line width=0.5mm,circle, below right of = W] (Y) {$Y$};
                \node[state,circle, above right of = Y] (X) {$X$};
                \node[state,circle, above right of = X] (Z) {$Z$};
                \node[state,blue,line width=0.5mm,circle, circle, above right of = W] (B) {$B$};
                \path (W) edge (Y);
                \path (B) edge (X);
                \path (X) edge (Y);
                \path (Z) edge (X);
                \path[bidirected] (Y) edge[bend right=45] (Z);
                \path[bidirected] (S) edge[bend right=45] (Y);
            \end{tikzpicture}
        }%
        \caption{$\graph_{\overline{B,W}}$; see \labelcref{eq:do_SB}.\label{fig:post_intervention2}}
    \end{subfigure}
    \caption{Causal diagram $\graph$ shown in \cref{fig:graph_two} with two post-intervention worlds in \cref{fig:post_intervention1} and \cref{fig:post_intervention2}. Blue node colour indicates the minimal set of variables which need to be measured to estimate the identifiable causal effect.}
    \label{fig:causal_graph_example_two}
\end{figure}
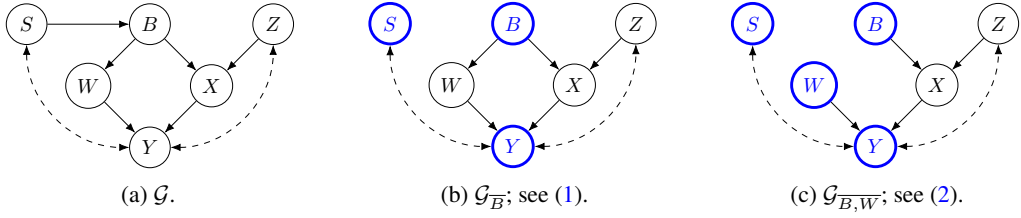

We demonstrate Def. \ref{def:mos} by considering the causal diagrams in \cref{fig:causal_graph_example_two}. Applying the rules of do-calculus we can express the interventional distributions in terms of observational mass functions:
\begin{align}
	\myP{Y \mid \DO{B}{b}}  &= \sum_{S}{\myP{Y \mid B, S}\myP{S}} = Q^{\set{B}}_Y \label{eq:do_X}  \\
	\myP{Y \mid \DO{B}{b},\DO{W}{w}} &= \sum_{S}{\myP{Y \mid B,W,S}\myP{S}}= Q^{\set{B,W}}_Y \label{eq:do_SB}
\end{align}
The \mos relative to $Q^{\set{B}}_Y$, shown on the right-hand-side (\rhs) in \labelcref{eq:do_X} is $\mathbf{O}^{\set{B}}_{\graph,Y} =  \{Y, B, S\}$ with the mutilated causal diagram shown in \cref{fig:post_intervention1}. Similarly for \cref{fig:post_intervention2} the \mos is given by $ \mathbf{O}^{\{B,W\}}_{\graph,Y} = \{Y,B,W,S\}$, with $Q^{\{B,W\}}_Y$ given on the \rhs in \labelcref{eq:do_SB}.


\section{Problem statement}\label{sec:problem_statement}

Consider a causal graph $\graph$ that encodes the causal relationship among a finite set of variables $\V$ in a stationary \scm $\scM = \scmdef$. We are interested in manipulating $\X \subseteq \V$ to minimise a target variable $Y \in \V \setminus \X$, which we assume is bounded, i.e. $|y| \leq M < \infty$ for some $M \in \mathbb{R}$ and all $y \in \text{dom}(Y)$. This objective is formally expressed as
\begin{align}
    \minimise_{
        \substack{\X' \in \mathbf{P}_{\graph,Y}^{\mathbf{V}}; \\ \x' \in \dom{\X'}}
    }\mu(\X',\x') & \triangleq \mathbb{E}[Y \mid \DO{\X'}{\x'}] \label{eq:cbo}
\end{align}
We assume that interventions are atomic (also known as `hard' \citep{model_based_cbo} or `perfect') as modelled by the do-operator \citep{pearl2000causality}. `Soft' or `stochastic' \citep{correa2020calculus} intervention settings are left for future work. We further assume that the functional relationships in $\scM$ (i.e. $\mat{F}$) are unknown (but $\graph$ is assumed known), which means that minimising \labelcref{eq:cbo} requires estimating $\expectation{}{Y \mid \DO{\X'}{\x'}}$ from data. This estimation can be done in two ways: \textit{i}) by intervening and conducting a controlled experiment, which yields samples from the interventional distribution $P(Y \mid \DO{\X'}{\x'})$; and \textit{ii}) by passively observing $\mathbf{O}^{\X'}_{\graph,Y}$ (see \cref{def:mos}) and using $\graph$ to estimate the causal effect through the do-calculus \citep{causality2009} (given that the causal effect is identifiable, otherwise the causal effect has to be estimated by intervening). Both estimation procedures are perturbed by additive Gaussian noise $\epsilon_t \sim \mathcal{N}(0,\sigma^2)$ and involve costs. Denote with $c(\X', e_t=\intervene)$ the cost of estimating $\mathbb{E}[Y \mid \DO{\X'}{\x'}]$ by \underline{intervening} and denote with $c(\X', e_t=\observe)$ the cost of estimating the same expression by \underline{observing}. The problem, then, is to design a sequence of interventions $(\DO{\X'_t}{\x'_{t}})_{t\in \{1,\hdots,T\}}$ and a sequence of estimation procedures $(e_t)_{t\in \{1,\hdots,T\}}$ to find an intervention that minimises \labelcref{eq:cbo} while keeping the cumulative cost below a maximum cost $K$. This problem can be formally stated as
\begin{subequations}\label{eq:formal_problem}
\begin{align}
  \minimise_{
  \substack{
  l_t,\X'_t,\x'_t\\
  t\in\{1,\hdots,T\}
  }}\quad
    & \left \{ \min_{t\in\{1,\hdots,T\}}\mu(\X'_t,\x'_t)-\mu(\X^*,\x^*) \right \} \label{eq:obj}\\
  \text{subject to }\quad& \sum_{t=1}^{T}c(\X'_t, e_t) < K, \quad\text{ } \X'_t \in \mathbf{P}_{\graph,Y}^{\mathbf{V}}, \quad\text{} \x'_t \in \dom{\X_t'}&\forall t \in \{1,\hdots,T\}\label{eq:cost_constraint}\\
    &e_t \in \begin{dcases}
    \{\intervene, \observe \} &\text{if }P(Y \mid \DO{\X'_{t}}{\x_t'}) \text{ is identifiable}\\
    \{\intervene \} & \text{otherwise}
    \end{dcases} & \forall t \in \{1,\hdots,T\}\label{eq:evaluatin_method}
\end{align}
\end{subequations}
where ($\X^*,\x^*$) denotes the minimiser of \labelcref{eq:cbo} and the expression inside the brackets of \labelcref{eq:obj} is the simple regret metric \cite{garnett_bayesoptbook_2023}. Further, \labelcref{eq:cost_constraint} -- \labelcref{eq:evaluatin_method} define the cost and domain constraints. The time-horizon $T$ is defined as the largest $t \in \mathbb{N}$ for which \labelcref{eq:cost_constraint} is satisfied.

The above problem is challenging for two reasons. First, to select the optimal intervention $\DO{\X'_t}{\x'_t}$ to evaluate at each stage $t \in \{1,\hdots,T\}$, it is necessary to take into account both \textit{exploration} (evaluating the causal effects in regions of high uncertainty) and \textit{exploitation} (evaluating the causal effects in regions deemed promising based on previous evaluations). Second, in selecting the evaluation procedures $(e_t)_{t \in \{1,\hdots,T\}}$, it is necessary to balance the trade-off between \textit{intervening} (estimating causal effects through controlled experimental evaluations) and \textit{observing} (estimating causal effects through the do-calculus).

The exploration-exploitation trade-off is well-studied in the statistical learning literature (see textbooks \citep{garnett_bayesoptbook_2023, lattimore2020bandit}) and numerous \textit{acquisition functions} that balance this trade-off have been proposed \cite{srinivas2012information,garnett_bayesoptbook_2023,lattimore2020bandit}. In contrast, the observation-intervention trade-off, which is the focus of this paper, is still relatively unexplored. In the following sections, we formulate this trade-off as an optimal stopping problem and present our main solution approach -- \emph{Optimal Stopping for Causal Optimisation}.

\section{Optimal stopping formulation of the observation-intervention trade-off}\label{sec:os_formulation}

\begin{wrapfigure}{r}{0.55\textwidth}
\vspace{-1em}
    \centering
    \includegraphics[width=0.54\textwidth]{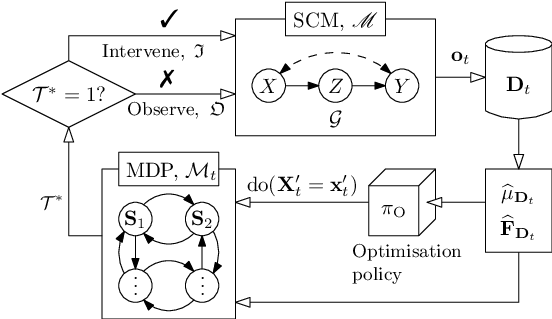}
    \caption{Optimal Stopping for Causal Optimisation (\osco) to balance the \textit{intervention-observation} trade-off; an optimisation policy $\pi_O$ decides on a sequence of interventions $(\DO{\X_t}{\x_t})_{t\in\{1,\hdots,T\}}$ to evaluate in an \scm $\mathscr{M}$ and the procedures to evaluate these interventions are decided by solving optimal stopping problems $(\mathcal{M}_t)_{t\in\{1,\hdots,T\}}$.}
     \label{fig:stopping_system}
\vspace{-1em}
\end{wrapfigure}

We formulate the problem of designing the sequence of estimation procedures $(e_t)_{t \in \{1,\hdots,T\}}$ as a series of \textit{optimal stopping} problems \cite{wald,shirayev,stopping_book_1,chow1971great}. In this formulation, we assume that an optimisation policy $\pi_{\mathrm{O}}$ that inspects the available data and selects the intervention $\DO{\X'_t}{\x'_t}$ to evaluate at each stage $t$, is given. We place no restrictions on how this policy is obtained or implemented. It may, for example, be derived from an acquisition function that balances the exploration-exploitation trade-off, as is done in e.g. \cbo \citep{cbo}. We further assume that the objective function $\mu$ \labelcref{eq:cbo} and the functions $\mat{F}$ of the underlying \scm are estimated by the probabilistic models $\widehat{\mu}_{\D_t}$ and $\widehat{\mat{F}}_{\D_t}$, respectively. Here $\D_t$ represents the available data at stage $t$ of the optimisation and as $|\D_t| \rightarrow \infty$, $\widehat{\mu}_{\D_t} \rightarrow \mu$ and $\widehat{\mat{F}}_{\D_t} \rightarrow \mat{F}$.

The  models $\widehat{\mu}$ and $\widehat{\mat{F}}$ allow us to guide the optimisation process and quantify the expected value and uncertainty in different regions of the interventional space \labelcref{eq:cbo}. Specifically, $\widehat{\mat{F}}_{\D_t}$ allows us to estimate causal effects through the do-calculus and $\widehat{\mu}_{\D_t}$ represents the current knowledge of the causal effects, allowing the optimisation policy $\pi_{\mathrm{O}}$ to make informed decisions about which intervention to evaluate at each stage.

Given the optimisation policy and the probabilistic models defined above, we seek to design the sequence $(e_t)_{t\in \{1,\hdots,T\}}$ to optimally allocate the available evaluation budget between intervening and observing so as to minimise \labelcref{eq:obj}. This task can be formally expressed as a series of Markovian and stationary optimal stopping problems $\mathcal{M}_1, \hdots, \mathcal{M}_T$ (see  \cref{fig:stopping_system}). To see this, note that, at any stage $t$ of the optimisation, the models $\widehat{\mu}_{\D_t}$ and $\widehat{\mat{F}}_{\D_t}$ allow us to simulate the growth of the dataset $\D_t$ and plan ahead when deciding between intervening and observing. This look-ahead planning involves two well-known challenges: \textit{i}) the possibly mis-specified models $\widehat{\mu}_{\D_t}$ and $\widehat{\mat{F}}_{\D_t}$ may lead to error-propagation when simulating many steps into the future \cite{nomyopic_bo_5}; and \textit{ii}) the number of possible simulation trajectories of $\D_t$ is infinite, which means that the planning problem corresponds to solving an intractable Markov Decision Process (\mdp) \cite{puterman,nomyopic_bo_1}. Most existing algorithms deal with these problems by truncating the planning horizon to one step \cite{garnett_bayesoptbook_2023,cbo,dcbo,model_based_cbo}. We propose to instead truncate the planning horizon to the next intervention, which may involve simulating many observation steps. This means that the growth of the dataset $\D_t$ follows a stationary Markov process governed by the probability law
\begin{align}
\mathbf{S}_1 &\triangleq \D_t, \quad \mathbf{S}_{T+1} \triangleq \bot, \quad \mathbf{S}_{k+1}\triangleq \mathbf{S}_{k}\cup \{\mathbf{o}_{k+1},k+1\} \text{ }(k<T), \quad \mathbf{o}_{k<T} \sim \widehat{\mat{F}}_{\D_t}\left (\mathbf{O}^{\mathbf{X}_t}_{\graph,Y} \right)\label{eq:stopping_dynamics}
\end{align}
where $\mathbf{S}_k \in \mathbf{S}$ denotes the state of the process at time-step $k$ and $\bot$ is an absorbing terminal state. At each time-step $k>1$ of this process, a new observation $\mathbf{o}_k$ is sampled from $\widehat{\mat{F}}_{\D_t}$ and added to the dataset $\D_t \cup \{\mathbf{o}_2,\hdots,\mathbf{o}_{k-1}\}$, which results in a new state $\mathbf{S}_{k+1}$. The process is stopped whenever an intervention ($e_t=\intervene$) is carried out. Thus the problem of deciding between intervening and observing becomes one of optimal stopping, where the goal is to find an optimal \textit{stopping time} $\mathcal{T}^{*}$:
\begin{align}
&\mathcal{T}^{*} \in \argmax_{\mathcal{T} \in \{1,\hdots,T\}}\left\{\gamma^{\mathcal{T}-1}r(\mathbf{S}_{\mathcal{T}})-\sum_{k=1}^{\mathcal{T}-1}\gamma^{k-1}c(\X'_t, \observe)\right\} & \text{subject to \labelcref{eq:cost_constraint} -- \labelcref{eq:evaluatin_method}}\label{eq:stopping_time_problem}
\end{align}
where $\mathcal{T} = \inf\{t: t \geq 1, \text{ }e_t=\intervene\}$ and $r(\mathbf{S}_{\mathcal{T}})$ denotes the reward of intervening (stopping) at time $\mathcal{T}$. Note that if the observation process has not been stopped at time $k=T-1$, the cost constraint in \labelcref{eq:cost_constraint} forces it to stop at time $T$, even if no intervention is carried out. We refer the reader to \cref{sec:mdp_background} for background on optimal stopping theory.

Due to the Markov property, the stopping problem can equivalently be formulated as an \mdp and any stopping time that satisfies \labelcref{eq:stopping_time_problem} is also a solution to the following Bellman equation \cite[Thm. 3.2]{chow1971great}
\begin{align}
\max_{e_t \in \{\intervene, \observe\}}\left\{r(\mathbf{S}_{k}), \int\hdots \int_{\dom{\mathbf{O}^{\mathbf{X}_t}_{\graph,Y}}}P_{\widehat{\mat{F}_{\D_t}}}(\mathbf{o}_{k}) V^{*}(\mathbf{S}_{k} \cup \{\mathbf{o}_{k+1}\}) \operatorname{d}\mathbf{O}^{\mathbf{X}_t}_{\graph,Y} - c(\X'_t, \observe)\right\}\label{eq:stopping_time_bellman}
\end{align}
subject to \labelcref{eq:cost_constraint} -- \labelcref{eq:evaluatin_method} and \labelcref{eq:stopping_dynamics}.

By solving \labelcref{eq:stopping_time_problem}, we obtain the optimal stopping time $\mathcal{T}^{*}$, which decides the next evaluation procedure $e_t$. In particular, if $\mathcal{T}^{*}=1$, the causal effect is estimated by intervening ($e_t=\intervene$) and otherwise the causal effect is estimated by observing ($e_t=\observe$). In either case, the resulting samples are used to update the probabilistic models $\widehat{\mu}_{\D_t}$ and $\widehat{\mat{F}}_{\D_t}$ and proceed to the next stage of the optimisation, wherein the next stopping problem $\mathcal{M}_{t+1}$ is defined. Note that a solution to \labelcref{eq:stopping_time_problem} always exists since $\mathcal{T}$ is restricted to the finite set $\{1,\hdots,T\}$, where $T < \infty$ \cite[Thm. 3.2]{chow1971great}.

\paragraph{The stopping reward.} A key issue in the design of the above stopping problem is the stopping reward $r(\mathbf{S}_{\mathcal{T}})$, which models how beneficial it is to intervene given the state $\mathbf{S}_{\mathcal{T}}$. An intervention can be beneficial to the optimisation in two ways. First, it can improve the current estimate of the optimum. Second, it can reduce the uncertainty in the objective function $\mu$. We model these two benefits with $\widehat{\mu}_{\mathbf{S}_k}$ and the \textit{information gain} measure $I$ \cite{Cover2006}, respectively:
\begin{align}
r(\mathbf{S}_T) &\triangleq \eta I(\mathbf{S}_k;\mu) \text{ and } r(\mathbf{S}_{k<T}) \triangleq \eta I(\mathbf{S}_k;\mu) -\kappa\widehat{\mu}_{\mathbf{S}_k}(\X_k', \x_{k}') -\tau\frac{\text{Vol}(\V)}{\text{Vol}(\mathbf{S}_{k})} - c(\X'_k, \intervene)\label{eq:stopping_reward}
\end{align}
Here $r(\bot)\triangleq 0$ and $\eta, \tau$ and $\kappa$ are scalar constants. The information gain $I(\mathbf{S}_k;\mu)=H(\mathbf{S}_k)-H(\mathbf{S}_k|\mu)$ quantifies the reduction in uncertainty about $\mu$ from revealing the dataset $\mathbf{S}_k$, where $H$ is the differential entropy function \cite{Cover2006,srinivas2012information}. The terms $\widehat{\mu}_{\mathbf{S}_k}(\X_k',\x_{k}')$ and $c(\X'_k, \intervene)$ quantify the expected value and the cost of the intervention, respectively. Finally, $\frac{\text{Vol}(\V)}{\text{Vol}(\mathbf{S}_k)}$ denotes the convex hull of the interventional domain of $\V$ divided by the convex hull of the observations in $\mathbf{S}_k$. The purpose of this term is to incentivise collection of observations in the beginning of the optimisation when $|\D_t|$ is small and it is not possible to plan ahead using the models $\widehat{\mu}$ and $\widehat{\mat{F}}$ (a similar term is used in \cite{cbo}).


\subsection{Efficient computation of the optimal stopping time}
\Cref{eq:stopping_time_bellman} implies that it is optimal to intervene (stop) whenever $r(\mathbf{S}_{k}) \geq \alpha_{\mathbf{S}_{k}}$, where $\alpha_{\mathbf{S}_{k}}$ denotes the second expression inside the maximisation in \labelcref{eq:stopping_time_bellman}. This means that we can divide the state space into two subsets defined by
\begin{align}
\mathscr{S}_{T-k} \triangleq \{\mathbf{S}_k \mid \mathbf{S}_k \in \mathbf{S}, r(\mathbf{S}_{k}) \geq \alpha_{\mathbf{S}_{k}}\} \quad\text{and}\quad \mathscr{C}_{T-k} \triangleq \{\mathbf{S}_k \mid \mathbf{S}_k \in \mathbf{S}, r(\mathbf{S}_{k}) < \alpha_{\mathbf{S}_{k}}\}\label{eq:stopping_sets}
\end{align}
where $\mathscr{S}_{l}$ and $\mathscr{C}_l$ are the stopping and continuation sets with $l$ time-steps remaining, respectively. These sets cannot overlap and their union $\mathscr{S}_l\cup \mathscr{C}_l$ covers the state space $\mathbf{S}$. Since the set of admissible stopping times in \labelcref{eq:stopping_time_problem} decreases as $l \rightarrow 1$, the stopping sets form an increasing sequence $\mathscr{S}_l \subseteq \mathscr{S}_{l-1} \subseteq \hdots \subseteq \mathscr{S}_{1}$ and similarly the continuation sets form a decreasing sequence $\mathscr{C}_l \supseteq \mathscr{C}_{l-1} \supseteq \hdots \supseteq \mathscr{C}_{1}$. Using these sets, the optimal stopping time can be expressed as
\begin{align}
\mathcal{T}^{*} &= \min\left\{k \mid k \in \mathbb{N}, \text{ }\mathbf{S}_{T-k} \in \mathscr{S}_k \right\}\label{eq:stopping_time_redefined}
\end{align}
Based on \labelcref{eq:stopping_sets}-\labelcref{eq:stopping_time_redefined} we state the following structural result regarding the optimal stopping times for the stopping problem defined in \labelcref{eq:stopping_time_problem}.
\begin{theorem}\label{thm:closed_stopping_set}
Given the stopping problem in \labelcref{eq:stopping_time_problem}, if a) the optimisation policy $\pi_{\mathrm{O}}$ is such that $\widehat{\mu}_{\mathbf{S}}(\pi_{\mathrm{O}}(\mathbf{S}))$ is supermodular and $c(\pi_{\mathrm{O}}(\mathbf{S}), \intervene)$ is non-increasing in $|\mathbf{S}|$; and b) $I(\mathbf{S}_k;\mu)$ is submodular, then $\mathscr{S}_{1}$ is closed. That is, $P(\mathbf{S}_{k+1}=\mathbf{s}_{k+1} \mid \mathbf{S}_{k}=\mathbf{s}_{k}) = 0$ if $\mathbf{s}_{k} \in \mathscr{S}_{1}$ and $\mathbf{s}_{k+1} \not\in \mathscr{S}_{1}$.
\end{theorem}
\begin{proof}
See \cref{sec:thm_1_proof}.
\end{proof}
Informally, \cref{thm:closed_stopping_set} states that if a state $\mathbf{s}_{k}$ is encountered for which it is better to intervene than to collect one more observation and then intervene, then no matter the next observation, the next state will always satisfy the same property. This result hinges on two assumptions. Assumption a) states, informally, that as the uncertainty about $\mu$ is reduced, the optimisation policy $\pi_{\mathrm{O}}$ explores less and instead prefers exploiting regions of the interventional space that are deemed promising based on $\widehat{\mu}_{\mathbf{S}}$. This assumption is for example satisfied by an $\epsilon$-greedy optimisation policy with decaying $\epsilon$. Similarly, the informal interpretation of assumption b) is that the gain of collecting observations reduces with the number of observations. The conditions for b) to hold are given in \cite[Prop. 2]{krause_submodular_ig} and are true in general. They hold for example if $\widehat{\mu}_{\mathbf{S}}$ is a Gaussian process (\gp) \cite{srinivas2012information}.

A direct consequence of \cref{thm:closed_stopping_set} is that the optimal stopping time can be obtained from a simple rule that is efficient to implement in practice, as stated in the following corollary.
\begin{corollary}\label{cr:olsa}
If assumptions a) and b) in Theorem \ref{thm:closed_stopping_set} hold, then the optimal stopping time is given by
\begin{align}
\mathcal{T}^{*} &= \min\left\{k \mid k \in \mathbb{N}, \text{ }r(\mathbf{S}_{k}) \geq \mathbb{E}_{\mathbf{o}_{k+1}}\left[r(\mathbf{S}_{k} \cup \{\mathbf{o}_{k+1}\})\right]-c(\X'_t, \observe) \right\}\label{eq:olsa}
\end{align}
and the stopping sets are all equal, i.e. $\mathscr{S}_1=\mathscr{S}_2=\hdots=\mathscr{S}_{T-1}$.
\end{corollary}
\begin{proof}
See Appendix \ref{sec:cr_1_proof}.
\end{proof}
Corollary \ref{cr:olsa} states that the stopping problem in \labelcref{eq:stopping_time_problem}, which characterises the observation-intervention trade-off for the optimisation problem in \labelcref{eq:formal_problem}, permits an optimal solution that is efficient to implement in practical algorithms. In the following sections, we compare this solution to existing approaches and explain how it can be integrated with existing algorithms for optimisation problems with causal structure (e.g. \cbo and causal \mab algorithms). The pseudo-code for integrating the optimal stopping problem with the existing algorithms is listed in \cref{alg:os_pseudocode} in \cref{sec:algorithms}.




\section{Related work}
\label{sec:related_work}
Problems of optimising decision variables arise in many settings, ranging from the control of physical and computer systems to managing entire economies \cite{puterman,powell}. Depending on the characteristics of the optimisation problem, different solution methods are appropriate (e.g. convex optimisation \cite{boyd_convex}, dynamic programming \cite{bert05} and black-box optimisation \cite{lattimore2020bandit,garnett_bayesoptbook_2023}). We limit the following discussion to related work that studies grey-box optimisation problems with known causal structure. This line of research can be divided into two main approaches: causal \bo and causal \mab{}s. %

\paragraph{Causal Bayesian optimisation.} The literature on \bo \cite{kushner_bo_orig,mockus_bo_orig} is extensive (see textbook \cite{garnett_bayesoptbook_2023} and survey \cite{shahriari2015taking}). Most of the prior work on \bo is focused on the black-box setting and ignores prior knowledge about the objective function (see the recent tutorial paper by \citet{astudillo2021thinking} for examples of prior knowledge). \bo problems with known causal structure are usually studied under the aegis of the \scm{} framework, as is the case in this paper and in \cite{cbo,dcbo,model_based_cbo,branchini23,bograph}. \citet{astudillo2021thinking} departs from this idea by instead leveraging function networks in place of \scm{}s. Another design choice which differ among existing works is the intervention model. This paper studies the hard intervention model, which is consistent with \cite{cbo,dcbo,branchini23}, but differ from \cite{astudillo2021thinking,model_based_cbo}, which study the soft intervention model. Further, all of the existing works (including this paper) except \cite{branchini23,bograph} have in common that they assume the causal structure to be known. \citet{branchini23} and \citet{bograph} do not make this assumption and instead explore techniques that combine causal discovery with causal \bo.

This work differs from the previous research on causal \bo in two main ways. First, we propose a solution to balance the observation-intervention trade-off that emerges in grey-box optimisation problems with causal structure. Existing works either ignore this trade-off or rely on heuristics to balance it. Second, we quantify the costs associated with collecting observations and estimating causal effects via the do-calculus, which previous works do not (they generally assume that observations can be collected without cost).
\paragraph{Causal multi-armed bandits.} Non-trivial dependencies amongst bandit arms are typically listed under \emph{structured} bandits. When that structure is explicitly causal the namesake follows \citep{lee2018structural, nair2021budgeted, lu2020regret}. The literature on causal bandits is richer than that of causal \bo{}. \citet{bareinboim2015bandits} were the first to explore the connection between causal reasoning and \mab algorithms. \citet{lattimore2016causal} and \citet{sen2017identifying} introduced methods for best-arm identification and non-trivial challenges that arise when unobserved confounders (\uc{}s) are present in the \scm, which is explored in \citep{lee2018structural,lee2020characterizing} where the authors introduce the notion of \pomis{}s for graphs with and without non-manipulative variables respectively. More recently in \citep{varici2022causal} the authors prove regret bounds for causal \mab{}s with linear \scm{}s, binary intervention domains and soft interventions. In the listed works thus far, the graph is assumed known. This assumption is relaxed in \citep{lu2021causal}. Another direction is \emph{budgeted} \mab{}s \citep{madani2004budgeted} where pulling an arm comes as a fixed cost and the agent has a finite budget which she has to spend judiciously to find the best arm subject to that limitation.

Similar to the existing work in causal \bo, the main differences between this paper and the previous work on causal \mab are a) that we propose a solution to the observation-intervention trade-off; and b) we quantify the costs associated with collecting observations. Further, to our knowledge, ours is the first study that combines the structured and the budgeted \mab{} approaches.

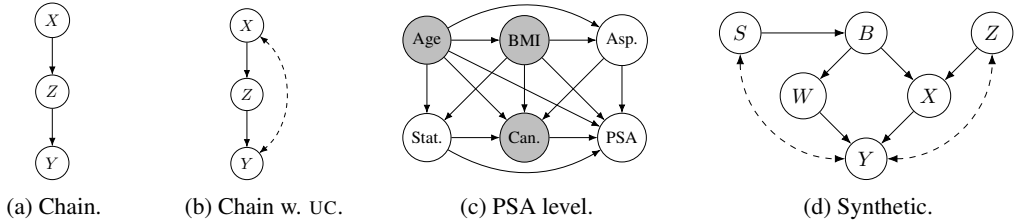
\begin{figure}[!htb]
  \centering
    \begin{subfigure}[t]{0.15\textwidth}
        \centering
        \resizebox{0.25\textwidth}{!}{%
           \begin{tikzpicture}[node distance =1.5cm]
                \node[state,circle] (X) {$X$};
                \node[state,circle, below of = X] (Z) {$Z$};
                \node[state,circle, below of = Z] (Y) {$Y$};
                \path (X) edge (Z);
                \path (Z) edge (Y);
            \end{tikzpicture}
        }%
        \caption{Chain.\label{fig:toygraph}}
    \end{subfigure}
    \hfill
    \begin{subfigure}[t]{0.15\textwidth}
        \centering
        \resizebox{0.44\textwidth}{!}{%
           \begin{tikzpicture}[node distance =1.5cm]
                \node[state,circle] (X) {$X$};
                \node[state,circle, below of = X] (Z) {$Z$};
                \node[state,circle, below of = Z] (Y) {$Y$};
                \path (X) edge (Z);
                \path (Z) edge (Y);
                \path[bidirected] (X) edge[bend left=50] (Y);
            \end{tikzpicture}
        }%
        \caption{Chain w. \uc{}.\label{fig:toygraph_w_uc}}
    \end{subfigure}
    \hfill
    \begin{subfigure}[t]{0.25\textwidth}
        \centering
        \resizebox{0.95\textwidth}{!}{%
           \begin{tikzpicture}[node distance =2cm]
                \node[state,fill=gray!50,circle] (A) {Age};                 
                \node[state,fill=gray!50,circle, right of = A] (B) {BMI};   
                \node[state,circle, right of = B] (C) {Asp.};                
                \node[state,circle, below of = A] (D) {Stat.};                
                \node[state,fill=gray!50,circle, right of = D] (E) {Can.};   
                \node[state,circle, right of = E] (F) {PSA};               

                \path (A) edge (B);
                \path (A) edge[bend left=30] (C);
                \path (A) edge (F);
                \path (A) edge (E);
                \path (A) edge (D);
                \path (B) edge (C);
                \path (B) edge (F);
                \path (B) edge (E);
                \path (B) edge (D);
                \path (C) edge (E);
                \path (C) edge (F);
                \path (D) edge (E);
                \path (D) edge[bend right=30] (F);
                \path (E) edge (F);
            \end{tikzpicture}
        }%
        \caption{PSA level.\label{fig:psa_level}}
    \end{subfigure}
    \hfill
    \begin{subfigure}[t]{0.3\textwidth}
        \centering
        \resizebox{0.95\textwidth}{!}{%
           \begin{tikzpicture}[node distance =1.5cm]
                \node[state,circle] (S) {$S$};
                \node[state,circle, below right of = S] (W) {$W$};
                \node[state,circle,circle, below right of = W] (Y) {$Y$};
                \node[state,circle, above right of = Y] (X) {$X$};
                \node[state,circle, above right of = X] (Z) {$Z$};
                \node[state, circle, above right of = W] (B) {$B$};
                \path (S) edge (B);
                \path (B) edge (W);
                \path (W) edge (Y);
                \path (B) edge (X);
                \path (X) edge (Y);
                \path (Z) edge (X);
                \path[bidirected] (Y) edge[bend right=45] (Z);
                \path[bidirected] (S) edge[bend right=45] (Y);
            \end{tikzpicture}
        }%
        \caption{Synthetic.\label{fig:synthetic_experiment}}
        \end{subfigure}
    \caption{Causal diagrams for the \scm{}s in the experimental evaluation; non-manipulative variables $\N \setminus \{Y\}$ are shaded and the outcome variable in each diagram is $Y \in \N$ apart from \cref{fig:psa_level} where the outcome variable is the PSA node.}
    \label{fig:experimental_DAGs}
\end{figure}

\section{Experimental evaluation}\label{sec:experiments}
We integrate \osco with state-of-the-art algorithms for optimisation problems with causal structure and evaluate these on a variety of synthetic and real-world \scm{s} with \DAG{}s given in \cref{fig:experimental_DAGs}.

\paragraph{Baselines.} Several algorithms for solving optimisation problems of the type defined in \labelcref{eq:cbo} (i.e. optimisation problems with causal structure) have emerged in recent years. These algorithms include \cbo \cite[Alg. 1]{cbo}, \mcbo \cite[Alg. 2]{model_based_cbo}, \dcbo \cite[Alg. 1]{dcbo}, \ceo \cite[Alg. 1]{branchini23}, the parallel bandit algorithm \cite[Alg. 1]{causal_bandits_5}, causal Thompson sampling \cite[Alg. 1]{causal_bandits_4,causal_bandits_2}, \textsc{c-ucb} \cite[Alg.1 ]{causal_bandits_3} and \textsc{kl-ucb} \cite[Alg. 3]{causal_bandits_1}. Among these algorithms, we choose to integrate our solution (\osco) with \cbo \cite[Alg. 1]{cbo}, \mcbo \cite[Alg. 2]{model_based_cbo} and \textsc{c-ucb} \cite[Alg.1 ]{causal_bandits_3} as those algorithms are most consistent with our problem setting and assumptions. We leave the integration of \osco with other algorithms to future work. We also compare \osco against three heuristic baselines: a) \textsc{intervene}, b) \textsc{observe} and c) \textsc{random}, which correspond to policies that a) always intervene, b) always observe and c) selects between intervening and observing randomly. These latter results can be found in the appendices.

\begin{figure}[!htb]
  \centering
\hspace{-1.5cm}
\begin{subfigure}[t]{0.45\textwidth}
  \centering
  \scalebox{1.2}{
    \includegraphics[width=0.95\linewidth]{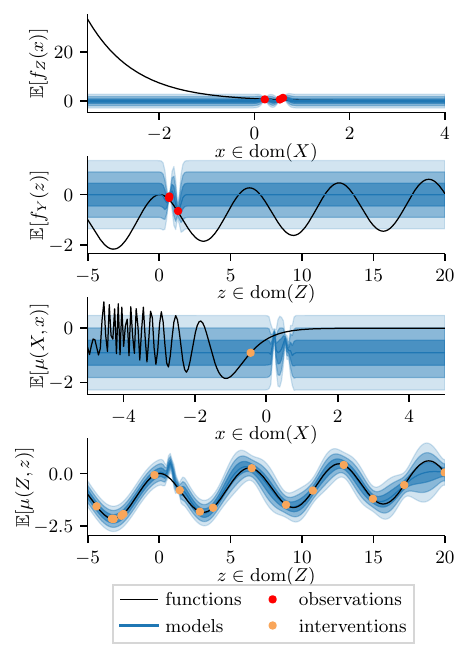}
    }
    \caption{\cbo with $\epsilon$-greedy (as used in \cite{cbo}).}
     \label{fig:cbo_gps_toygraph}
\end{subfigure}
\hspace{7mm}%
\begin{subfigure}[t]{0.45\textwidth}
  \centering
  \scalebox{1.2}{
    \includegraphics[width=0.95\linewidth]{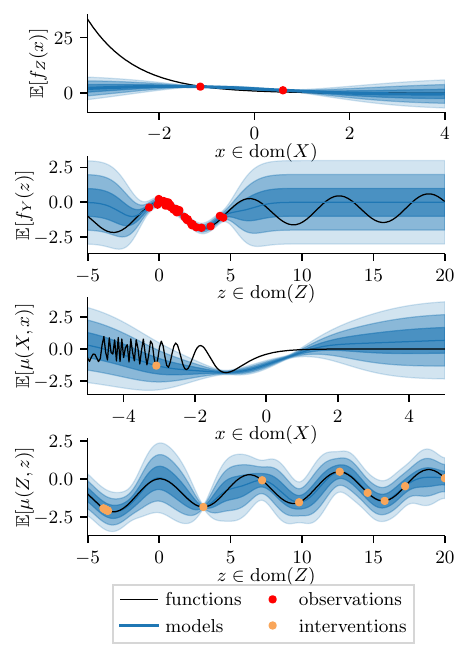}
    }
    \caption{\cbo with \osco.}
     \label{fig:cbo_os_gps_toygraph}
\end{subfigure}
 \caption{Collected data and estimated models from running \cbo with and without \osco to solve \labelcref{eq:cbo} for the chain \scm \cite[Fig. 3]{cbo} (see \cref{fig:toygraph}); the blue curves and the shaded blue areas show the mean and standard deviation of the estimated models $\widehat{\mat{F}}$ and $\widehat{\mu}$; the red and orange dots show observations and interventions respectively; the black lines show the \scm functions $\mat{F}$ and the causal effects $\mu$.}
 \label{fig:toygraph_gps}
\end{figure}

\begin{figure}
  \centering
    \includegraphics[width=1\linewidth]{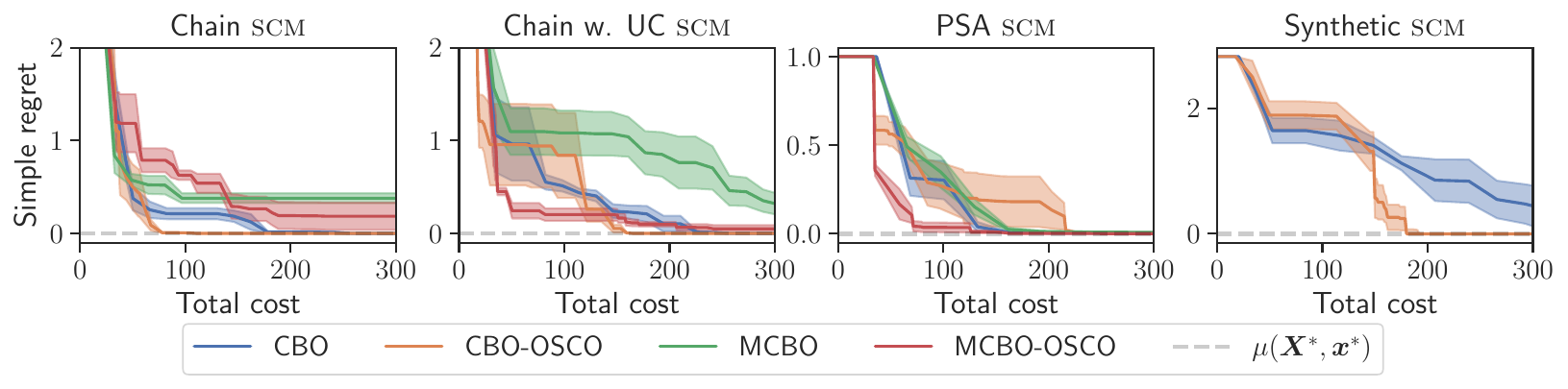}
    \begin{subfigure}[t]{0.5\textwidth}
        \centering
        \includegraphics[width=\linewidth]{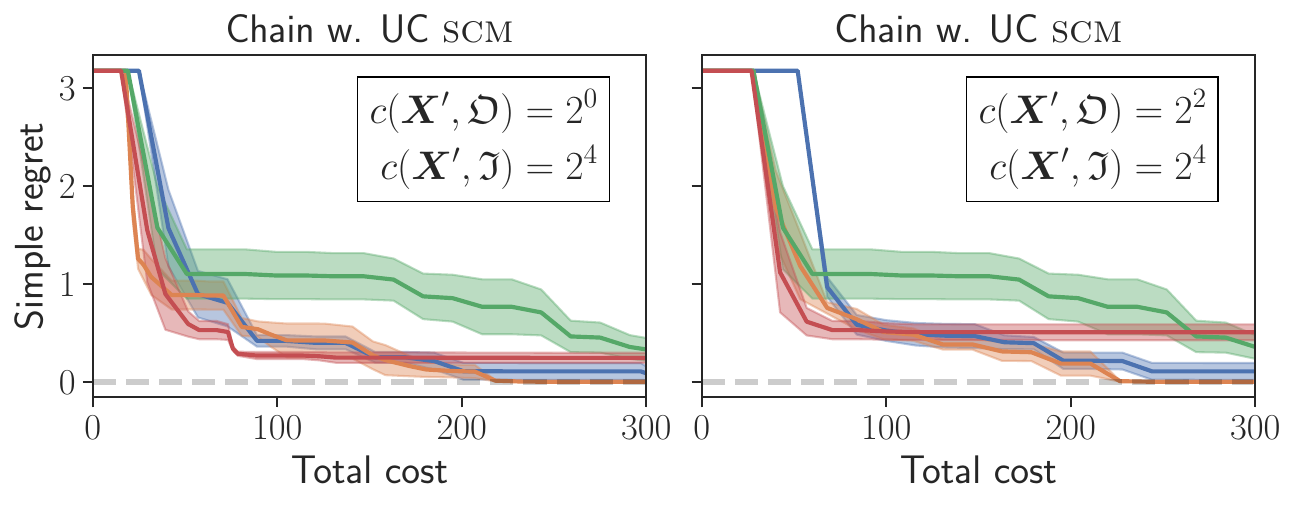}
    \end{subfigure}
    \begin{subfigure}[t]{0.24\textwidth}
        \centering
        \includegraphics[width=1.18\linewidth]{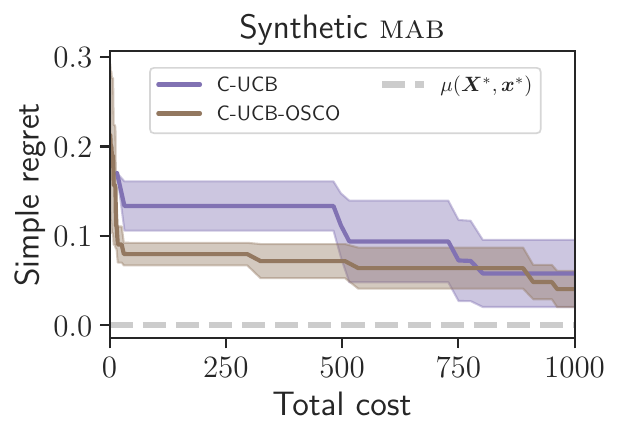}
    \end{subfigure}
    \caption{Convergence curves and computational overhead for \cbo, \mcbo and \textsc{c-ucb} with and without \osco for different \scm{}s and causal \mab{}s; the curves indicate the mean and standard deviation ($\pm \frac{\sigma}{\sqrt[]{3}}$) over three evaluations with different random seeds; the columns from left to right on the \textbf{first row} relate to the \scm{}s in \cref{fig:toygraph}, \cref{fig:toygraph_w_uc}, \cref{fig:psa_level} and \cref{fig:synthetic_experiment} respectively; the first two columns on the \textbf{second row} show convergence curves for different observation costs (all other curves are based on the observation cost $c(\X^{'}, \observe)=2^{-2}$) and the right-most column concerns the bandit version of the synthetic \scm in \cref{fig:synthetic_experiment}.
    }
     \label{fig:mega_convergence}
     \vspace{-1em}
\end{figure}

\paragraph{Experiment setup.} We run all experiments with three different random seeds and show the convergence curves of the simple regret metric $\argmin_{k \in \{1,\hdots,t\}}\mu(\X_k, \x_k)-\mu(\X^{*}, \x^{*}_k)$ \labelcref{eq:cbo} for $t=1,\hdots,T$ \cite{garnett_bayesoptbook_2023} (this is consistent with the metric reported in \cite{cbo} but differs from \cite{model_based_cbo,causal_bandits_3} which reports the cumulative regret metric \cite{lattimore2020bandit}). Hyperparameters are listed in \cref{sec:hyperparameters} and were chosen based on cross-validation. Throughout, we assume that the dataset at the start of the optimisation is empty ($\D_1=\emptyset$). This contrasts with \cite{cbo} and \cite{model_based_cbo}, which assume $|\D_1| > 0$ in all experiments. When we compare against \cbo we utilise \mos (\cref{def:mos}) to reduce the observation costs. We do not utilise \mos when we compare against \mcbo as \mcbo relies on complete observations. Further, we only evaluate \mcbo on the \scm{}s available in the official implementation \cite{model_based_cbo}, namely the chain \scm{} and the PSA \scm{}. In all implementations of \osco, we implement the stopping rule implied in Corollary \ref{cr:olsa} (even in cases when the assumptions of the corollary do not hold, in which case it provides an approximation of the optimal stopping time). Complete experimental details can be found in the appendix.

\paragraph{Results discussion.}

\Cref{fig:toygraph_gps} shows the estimated probabilistic models $\widehat{\mu}$ and $\widehat{\mat{F}}$ when running \cbo with two different policies for balancing the intervention-observation trade-off: \textit{i}) the $\epsilon$-greedy policy used in \cite{cbo}; and \textit{ii}) the \osco approach described in \cref{sec:os_formulation}. We note in the lowest plots that the optimal intervention is $\DO{Z}{-3.20}$ (with target value $Y=-2.17$) and that this intervention is found in both cases. We further note that \cbo with \osco is able to accurately estimate the interventional distributions through the do-calculus. The main differences between \cbo with $\epsilon$-greedy and \cbo with \osco are a) \cbo with $\epsilon$-greedy collects only $3$ observations, spending most of the evaluation budget on interventions, whereas \cbo with \osco uses most of the evaluation budget to collect observations; and b) that \cbo with $\epsilon$-greedy observes all endogenous variables whereas  \cbo with \osco only observes the \mos $\mathbf{O}^{Z}_{\graph,Y}=\{Z,Y\}$. That \cbo uses most of the evaluation budget on intervening whereas \cbo with \osco uses most of the budget on observations can be explained by two main reasons. First, the definition of $\epsilon$ in \cite[Eq. 6]{cbo} implies that the probability of observing in \cbo with $\epsilon$-greedy is close to $0$ when the number of previously collected observations is low. Second, the optimal stopping formulation in \labelcref{eq:olsa} implies that \cbo with \osco will observe rather than intervene when it is more cost-effective.

\Cref{fig:mega_convergence} compares \osco with baselines. The first row in \cref{fig:mega_convergence} shows convergence curves of \cbo and \mcbo with and without \osco for the chain \scm (\cref{fig:toygraph} and \cref{fig:toygraph_w_uc} -- with and without an \uc respectively), the synthetic \scm (\cref{fig:synthetic_experiment}) and the PSA \scm (\cref{fig:psa_level}). An ablation study for different observation costs is shown in the two left-most plots of the second row of \cref{fig:mega_convergence}. The right-most plots in the second row of \cref{fig:mega_convergence} show a) converge curves of \textsc{c-ucb} with and without \osco for the bandit version of the synthetic \scm (\cref{fig:synthetic_experiment}); and b) the computational overhead of \osco. We note that \cbo with \osco outperforms \cbo and finds the optimal intervention for all \scm{}s within the prescribed evaluation budget. Similarly, we observe that \textsc{c-ucb} with \osco is more cost-efficient than \textsc{c-ucb} without \osco for the synthetic \scm. We explain the efficient convergence of \osco by its design, which a) uses look-ahead-planning to decide between observing and intervening based on what is most cost effective; and b) utilises \mos (\cref{def:mos}) to limit the variables that need to be observed. We further note that \mcbo with \osco outperforms \mcbo on the chain and PSA \scm{}s. Moreover, we observe that the performance of \cbo is better than that of \mcbo on average, which is consistent with the results reported in \cite{model_based_cbo}. This result can be explained by the design of \mcbo, which is optimised for the cumulative regret metric rather than the simple regret. Finally, we observe that the computational overhead of \osco per iteration is less than a factor of $2$. Extended evaluation results can be found in \cref{sec:ablation}.

\section{Conclusion}
\label{sec:conclusion}
\begin{wrapfigure}{r}{0.5\textwidth}
        \vspace{-6em}
        \centering
        \includegraphics[width=0.5\textwidth]{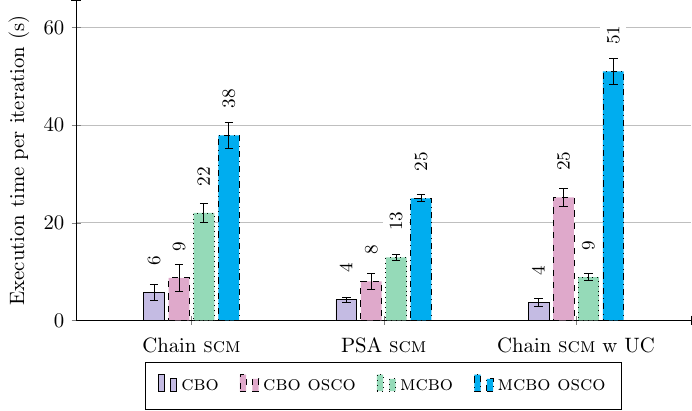}
        \caption{Execution times per iteration with and without \osco for some of the example \scm{}s. The height of each bar indicates the mean execution time from $100$ measurements and the error bars indicate the standard deviations.}
        \vspace{-1em}
\end{wrapfigure}

We have formally defined the observation-intervention trade-off that emerges in optimisation problems with causal structure and have shown that this trade-off can be formulated as a non-myopic optimal stopping problem whose solution determines when a causal effect should be estimated by intervening and when it is more cost-effective to collect observational data. We have also characterised the minimal set of variables that need to be observed to estimate the causal effect -- the minimal observation set (\mos). Extensive evaluation results on real and synthetic \scm{}s show that the optimal stopping formulation can enhance existing algorithms and that the computational overhead is manageable. This paper opens up several directions for future research. One direction is to extend our model to include soft interventions and longer planning horizons. Another direction is to evaluate different reward functions in the optimal stopping problem.




\clearpage
\bibliographystyle{icml2022}
\bibliography{references}

\begin{thebibliography}{68}
\providecommand{\natexlab}[1]{#1}
\providecommand{\url}[1]{\texttt{#1}}
\expandafter\ifx\csname urlstyle\endcsname\relax
  \providecommand{\doi}[1]{doi: #1}\else
  \providecommand{\doi}{doi: \begingroup \urlstyle{rm}\Url}\fi

\bibitem[{ABIM Foundation}(2014)]{blood_test}
{ABIM Foundation}.
\newblock {PSA Blood Test for Prostate Cancer}.
\newblock
  \url{https://www.choosingwisely.org/patient-resources/psa-test-for-prostate-cancer/},
  2014.
\newblock Accessed: 2023-03-23.

\bibitem[Aglietti et~al.(2020)Aglietti, Lu, Paleyes, and González]{cbo}
Aglietti, V., Lu, X., Paleyes, A., and González, J.
\newblock Causal {B}ayesian {O}ptimization.
\newblock In \emph{Proceedings of the Twenty Third International Conference on
  Artificial Intelligence and Statistics}, volume 108 of \emph{Proceedings of
  Machine Learning Research}, pp.\  3155--3164. PMLR, 26--28 Aug 2020.

\bibitem[Aglietti et~al.(2021)Aglietti, Dhir, Gonz\'{a}lez, and Damoulas]{dcbo}
Aglietti, V., Dhir, N., Gonz\'{a}lez, J., and Damoulas, T.
\newblock {Dynamic Causal Bayesian Optimization}.
\newblock In \emph{Advances in Neural Information Processing Systems},
  volume~35, 2021.

\bibitem[Alabed \& Yoneki(2022)Alabed and Yoneki]{bograph}
Alabed, S. and Yoneki, E.
\newblock Bograph: Structured bayesian optimization from logs for expensive
  systems with many parameters.
\newblock In \emph{Proceedings of the 2nd European Workshop on Machine Learning
  and Systems}, EuroMLSys '22, pp.\  45–53, New York, NY, USA, 2022.
  Association for Computing Machinery.
\newblock ISBN 9781450392549.
\newblock \doi{10.1145/3517207.3526977}.

\bibitem[Astudillo \& Frazier(2021)Astudillo and
  Frazier]{astudillo2021thinking}
Astudillo, R. and Frazier, P.~I.
\newblock Thinking inside the box: A tutorial on grey-box bayesian
  optimization.
\newblock In \emph{2021 Winter Simulation Conference (WSC)}, pp.\  1--15. IEEE,
  2021.

\bibitem[Bareinboim \& Pearl(2012)Bareinboim and Pearl]{bareinboim2012causal}
Bareinboim, E. and Pearl, J.
\newblock Causal inference by surrogate experiments: z-identifiability.
\newblock \emph{arXiv preprint arXiv:1210.4842}, 2012.

\bibitem[Bareinboim \& Pearl(2016)Bareinboim and Pearl]{bareinboim2016causal}
Bareinboim, E. and Pearl, J.
\newblock Causal inference and the data-fusion problem.
\newblock \emph{Proceedings of the National Academy of Sciences}, 113\penalty0
  (27):\penalty0 7345--7352, 2016.

\bibitem[Bareinboim et~al.(2015{\natexlab{a}})Bareinboim, Forney, and
  Pearl]{bareinboim2015bandits}
Bareinboim, E., Forney, A., and Pearl, J.
\newblock Bandits with unobserved confounders: A causal approach.
\newblock \emph{Advances in Neural Information Processing Systems},
  28:\penalty0 1342--1350, 2015{\natexlab{a}}.

\bibitem[Bareinboim et~al.(2015{\natexlab{b}})Bareinboim, Forney, and
  Pearl]{causal_bandits_4}
Bareinboim, E., Forney, A., and Pearl, J.
\newblock Bandits with unobserved confounders: A causal approach.
\newblock In Cortes, C., Lawrence, N., Lee, D., Sugiyama, M., and Garnett, R.
  (eds.), \emph{Advances in Neural Information Processing Systems}, volume~28.
  Curran Associates, Inc., 2015{\natexlab{b}}.

\bibitem[Bather(2000)]{bather_decision_theory}
Bather, J.
\newblock \emph{Decision Theory: An Introduction to Dynamic Programming and
  Sequential Decisions}.
\newblock John Wiley and Sons, Inc., USA, 2000.
\newblock ISBN 0471976490.

\bibitem[Bellman(1957{\natexlab{a}})]{bellman1957markovian}
Bellman, R.
\newblock A markovian decision process.
\newblock \emph{Journal of Mathematics and Mechanics}, 6\penalty0 (5):\penalty0
  679--684, 1957{\natexlab{a}}.

\bibitem[Bellman(1957{\natexlab{b}})]{bellman_eq}
Bellman, R.
\newblock \emph{{Dynamic Programming}}.
\newblock Dover Publications, 1957{\natexlab{b}}.
\newblock ISBN 9780486428093.

\bibitem[Bertsekas(2005)]{bert05}
Bertsekas, D.~P.
\newblock \emph{Dynamic Programming and Optimal Control}, volume~I.
\newblock Athena Scientific, Belmont, MA, USA, 3rd edition, 2005.

\bibitem[Boyd \& Vandenberghe(2004)Boyd and Vandenberghe]{boyd_convex}
Boyd, S. and Vandenberghe, L.
\newblock \emph{Convex Optimization}.
\newblock {Cambridge University Press}, March 2004.
\newblock ISBN 0521833787.

\bibitem[Branchini et~al.(2023)Branchini, Aglietti, Dhir, and
  Damoulas]{branchini23}
Branchini, N., Aglietti, V., Dhir, N., and Damoulas, T.
\newblock Causal entropy optimization.
\newblock In \emph{Proceedings of The 26th International Conference on
  Artificial Intelligence and Statistics}, volume 206 of \emph{Proceedings of
  Machine Learning Research}, pp.\  8586--8605. PMLR, 25--27 Apr 2023.

\bibitem[Chow et~al.(1971)Chow, Robbins, and Siegmund]{chow1971great}
Chow, Y., Robbins, H., and Siegmund, D.
\newblock Great expectations: The theory of optimal stopping.
\newblock In \emph{Journal of the Royal Statistical Society}, 1971.

\bibitem[Correa \& Bareinboim(2020)Correa and Bareinboim]{correa2020calculus}
Correa, J. and Bareinboim, E.
\newblock A calculus for stochastic interventions: Causal effect identification
  and surrogate experiments.
\newblock In \emph{Proceedings of the AAAI conference on artificial
  intelligence}, volume~34, pp.\  10093--10100, 2020.

\bibitem[Cover \& Thomas(2006)Cover and Thomas]{Cover2006}
Cover, T.~M. and Thomas, J.~A.
\newblock \emph{Elements of Information Theory 2nd Edition (Wiley Series in
  Telecommunications and Signal Processing)}.
\newblock Wiley-Interscience, July 2006.
\newblock ISBN 0471241954.

\bibitem[Evans et~al.(2000)Evans, Swartz, and Swartz]{evans2000approximating}
Evans, M., Swartz, T., and Swartz, A.
\newblock \emph{Approximating Integrals Via Monte Carlo and Deterministic
  Methods}.
\newblock Oxford statistical science series. Oxford University Press, 2000.
\newblock ISBN 9780198502784.

\bibitem[Ferro et~al.(2015)Ferro, Pina, Severo, Dias, Botelho, and
  Lunet]{ferro2015use}
Ferro, A., Pina, F., Severo, M., Dias, P., Botelho, F., and Lunet, N.
\newblock Use of statins and serum levels of prostate specific antigen.
\newblock \emph{Acta Urol{\'o}gica Portuguesa}, 32\penalty0 (2):\penalty0
  71--77, 2015.

\bibitem[{Frank DiVincenzo, k-health}(2022)]{cbc_test_cost}
{Frank DiVincenzo, k-health}.
\newblock {How much does blood work cost in 2022?}
\newblock
  \url{https://khealth.com/learn/healthcare/how-much-does-bloodwork-cost/},
  2022.
\newblock Accessed: 2023-04-17.

\bibitem[Garnett(2023)]{garnett_bayesoptbook_2023}
Garnett, R.
\newblock \emph{{Bayesian Optimization}}.
\newblock Cambridge University Press, 2023.
\newblock to appear.

\bibitem[GPyOpt(2016)]{gpyopt2016}
GPyOpt.
\newblock Gpyopt: A bayesian optimization framework in python.
\newblock \url{http://github.com/SheffieldML/GPyOpt}, 2016.

\bibitem[Hammar \& Stadler(2022)Hammar and Stadler]{hammar_stadler_tnsm}
Hammar, K. and Stadler, R.
\newblock Intrusion prevention through optimal stopping.
\newblock \emph{IEEE Transactions on Network and Service Management},
  19\penalty0 (3):\penalty0 2333--2348, 2022.
\newblock \doi{10.1109/TNSM.2022.3176781}.

\bibitem[Huang \& Valtorta(2006)Huang and Valtorta]{huang2006identifiability}
Huang, Y. and Valtorta, M.
\newblock Identifiability in causal bayesian networks: A sound and complete
  algorithm.
\newblock In \emph{AAAI}, pp.\  1149--1154, 2006.

\bibitem[Imber et~al.(2020)Imber, Varghese, Ehdaie, and
  Gorovets]{imber2020financial}
Imber, B.~S., Varghese, M., Ehdaie, B., and Gorovets, D.
\newblock Financial toxicity associated with treatment of localized prostate
  cancer.
\newblock \emph{Nature Reviews Urology}, 17\penalty0 (1):\penalty0 28--40,
  2020.

\bibitem[{Jacqueline Slobin, TalktoMira}(2022-08-22)]{statin_cost}
{Jacqueline Slobin, TalktoMira}.
\newblock {What's the Least Expensive Cholesterol Medication in 2021?}
\newblock
  \url{https://www.talktomira.com/post/how-much-do-statins-cost-without-insurance},
  2022-08-22.
\newblock Accessed: 2023-04-17.

\bibitem[Krause \& Guestrin(2005)Krause and Guestrin]{krause_submodular_ig}
Krause, A. and Guestrin, C.
\newblock Near-optimal nonmyopic value of information in graphical models.
\newblock In \emph{UAI}, pp.\  324--331. AUAI Press, 2005.
\newblock ISBN 0-9749039-1-4.

\bibitem[Kroon et~al.(2022)Kroon, Mooij, and Belgrave]{causal_bandits_2}
Kroon, A.~D., Mooij, J., and Belgrave, D.
\newblock Causal bandits without prior knowledge using separating sets.
\newblock In Schölkopf, B., Uhler, C., and Zhang, K. (eds.), \emph{Proceedings
  of the First Conference on Causal Learning and Reasoning}, volume 177 of
  \emph{Proceedings of Machine Learning Research}, pp.\  407--427. PMLR, 11--13
  Apr 2022.

\bibitem[Kushner(1962)]{kushner_bo_orig}
Kushner, H.~J.
\newblock A versatile stochastic model of a function of unknown and time
  varying form.
\newblock \emph{Journal of Mathematical Analysis and Applications}, 5\penalty0
  (1):\penalty0 150--167, 1962.
\newblock ISSN 0022-247X.
\newblock \doi{https://doi.org/10.1016/0022-247X(62)90011-2}.

\bibitem[Lattimore et~al.(2016{\natexlab{a}})Lattimore, Lattimore, and
  Reid]{causal_bandits_5}
Lattimore, F., Lattimore, T., and Reid, M.~D.
\newblock Causal bandits: Learning good interventions via causal inference.
\newblock In \emph{Proceedings of the 30th International Conference on Neural
  Information Processing Systems}, NIPS'16, pp.\  1189–1197, Red Hook, NY,
  USA, 2016{\natexlab{a}}. Curran Associates Inc.
\newblock ISBN 9781510838819.

\bibitem[Lattimore et~al.(2016{\natexlab{b}})Lattimore, Lattimore, and
  Reid]{lattimore2016causal}
Lattimore, F., Lattimore, T., and Reid, M.~D.
\newblock Causal bandits: Learning good interventions via causal inference.
\newblock In \emph{Advances in Neural Information Processing Systems}, pp.\
  1181--1189, 2016{\natexlab{b}}.

\bibitem[Lattimore \& Szepesv{\'a}ri(2020)Lattimore and
  Szepesv{\'a}ri]{lattimore2020bandit}
Lattimore, T. and Szepesv{\'a}ri, C.
\newblock \emph{Bandit algorithms}.
\newblock Cambridge University Press, 2020.

\bibitem[Lee \& Bareinboim(2018)Lee and Bareinboim]{lee2018structural}
Lee, S. and Bareinboim, E.
\newblock Structural causal bandits: where to intervene?
\newblock \emph{Advances in Neural Information Processing Systems}, 31, 2018.

\bibitem[Lee \& Bareinboim(2019)Lee and Bareinboim]{causal_bandits_1}
Lee, S. and Bareinboim, E.
\newblock Structural causal bandits with non-manipulable variables.
\newblock \emph{Proceedings of the AAAI Conference on Artificial Intelligence},
  33\penalty0 (01):\penalty0 4164--4172, Jul. 2019.
\newblock \doi{10.1609/aaai.v33i01.33014164}.

\bibitem[Lee \& Bareinboim(2020)Lee and Bareinboim]{lee2020characterizing}
Lee, S. and Bareinboim, E.
\newblock Characterizing optimal mixed policies: Where to intervene and what to
  observe.
\newblock \emph{Advances in neural information processing systems}, 33, 2020.

\bibitem[Lu et~al.(2020{\natexlab{a}})Lu, Meisami, Tewari, and
  Yan]{causal_bandits_3}
Lu, Y., Meisami, A., Tewari, A., and Yan, W.
\newblock Regret analysis of bandit problems with causal background knowledge.
\newblock In Peters, J. and Sontag, D. (eds.), \emph{Proceedings of the 36th
  Conference on Uncertainty in Artificial Intelligence (UAI)}, volume 124 of
  \emph{Proceedings of Machine Learning Research}, pp.\  141--150. PMLR, 03--06
  Aug 2020{\natexlab{a}}.

\bibitem[Lu et~al.(2020{\natexlab{b}})Lu, Meisami, Tewari, and
  Yan]{lu2020regret}
Lu, Y., Meisami, A., Tewari, A., and Yan, W.
\newblock Regret analysis of bandit problems with causal background knowledge.
\newblock In \emph{Conference on Uncertainty in Artificial Intelligence}, pp.\
  141--150. PMLR, 2020{\natexlab{b}}.

\bibitem[Lu et~al.(2021)Lu, Meisami, and Tewari]{lu2021causal}
Lu, Y., Meisami, A., and Tewari, A.
\newblock Causal bandits with unknown graph structure.
\newblock \emph{Advances in Neural Information Processing Systems},
  34:\penalty0 24817--24828, 2021.

\bibitem[Madani et~al.(2004)Madani, Lizotte, and Greiner]{madani2004budgeted}
Madani, O., Lizotte, D.~J., and Greiner, R.
\newblock The budgeted multi-armed bandit problem.
\newblock In \emph{International Conference on Computational Learning Theory},
  pp.\  643--645. Springer, 2004.

\bibitem[Mo{\v{c}}kus(1975)]{mockus_bo_orig}
Mo{\v{c}}kus, J.
\newblock On bayesian methods for seeking the extremum.
\newblock In Marchuk, G.~I. (ed.), \emph{Optimization Techniques IFIP Technical
  Conference Novosibirsk, July 1--7, 1974}, pp.\  400--404, Berlin, Heidelberg,
  1975. Springer Berlin Heidelberg.
\newblock ISBN 978-3-540-37497-8.

\bibitem[Nair et~al.(2021)Nair, Patil, and Sinha]{nair2021budgeted}
Nair, V., Patil, V., and Sinha, G.
\newblock Budgeted and non-budgeted causal bandits.
\newblock In \emph{International Conference on Artificial Intelligence and
  Statistics}, pp.\  2017--2025. PMLR, 2021.

\bibitem[Pate et~al.(2014)Pate, Uhlman, Rosenthal, Cram, and
  Erickson]{pate2014variations}
Pate, S.~C., Uhlman, M.~A., Rosenthal, J.~A., Cram, P., and Erickson, B.~A.
\newblock {Variations in the open market costs for prostate cancer surgery: a
  survey of US hospitals}.
\newblock \emph{Urology}, 83\penalty0 (3):\penalty0 626--631, 2014.

\bibitem[Pearl(2009{\natexlab{a}})]{causality2009}
Pearl, J.
\newblock \emph{Causality: Models, Reasoning and Inference}.
\newblock Cambridge University Press, USA, 2nd edition, 2009{\natexlab{a}}.

\bibitem[Pearl(2009{\natexlab{b}})]{pearl2000causality}
Pearl, J.
\newblock \emph{Causality: models, reasoning and inference}.
\newblock Cambridge university press, 2009{\natexlab{b}}.

\bibitem[Pearl \& Robins(1995)Pearl and Robins]{pearl1995probabilistic}
Pearl, J. and Robins, J.~M.
\newblock Probabilistic evaluation of sequential plans from causal models with
  hidden variables.
\newblock In \emph{UAI}, volume~95, pp.\  444--453. Citeseer, 1995.

\bibitem[Peskir \& Shiryaev(2006)Peskir and Shiryaev]{stopping_book_1}
Peskir, G. and Shiryaev, A.
\newblock \emph{{Optimal stopping and free-boundary problems}}.
\newblock Lectures in mathematics (ETH Zürich). Springer, 2006.

\bibitem[Powell(2011)]{powell}
Powell, W.~B.
\newblock \emph{Approximate Dynamic Programming: Solving the Curses of
  Dimensionality}.
\newblock Wiley Series in Probability and Statistics. Wiley, Hoboken, NJ, USA,
  2nd edition, 2011.

\bibitem[Puterman(1994)]{puterman}
Puterman, M.~L.
\newblock \emph{Markov Decision Processes: Discrete Stochastic Dynamic
  Programming}.
\newblock John Wiley and Sons, Inc., USA, 1st edition, 1994.
\newblock ISBN 0471619779.

\bibitem[Ross(1983)]{ross_stochastic_dp}
Ross, S.~M.
\newblock \emph{Introduction to Stochastic Dynamic Programming: Probability and
  Mathematical}.
\newblock Academic Press, Inc., USA, 1983.
\newblock ISBN 0125984200.

\bibitem[Rubinstein \& Kroese(2016)Rubinstein and Kroese]{rubinstein_mc}
Rubinstein, R.~Y. and Kroese, D.~P.
\newblock \emph{Simulation and the Monte Carlo Method}.
\newblock Wiley Publishing, 3rd edition, 2016.
\newblock ISBN 1118632168.

\bibitem[Sen et~al.(2017)Sen, Shanmugam, Dimakis, and
  Shakkottai]{sen2017identifying}
Sen, R., Shanmugam, K., Dimakis, A.~G., and Shakkottai, S.
\newblock Identifying best interventions through online importance sampling.
\newblock In \emph{International Conference on Machine Learning}, pp.\
  3057--3066. PMLR, 2017.

\bibitem[Shahriari et~al.(2015)Shahriari, Swersky, Wang, Adams, and
  De~Freitas]{shahriari2015taking}
Shahriari, B., Swersky, K., Wang, Z., Adams, R.~P., and De~Freitas, N.
\newblock Taking the human out of the loop: A review of bayesian optimization.
\newblock \emph{Proceedings of the IEEE}, 104\penalty0 (1):\penalty0 148--175,
  2015.

\bibitem[Shirayev(2007)]{shirayev}
Shirayev, A.~N.
\newblock \emph{Optimal Stopping Rules}.
\newblock Springer-Verlag Berlin, 2007.
\newblock Reprint of russian edition from 1969.

\bibitem[Shpitser \& Pearl(2006{\natexlab{a}})Shpitser and
  Pearl]{shpitser2006conditional}
Shpitser, I. and Pearl, J.
\newblock Identification of conditional interventional distributions.
\newblock In \emph{Proceedings of the Twenty-Second Conference on Uncertainty
  in Artificial Intelligence}, pp.\  437–444, Arlington, Virginia, USA,
  2006{\natexlab{a}}. AUAI Press.

\bibitem[Shpitser \& Pearl(2006{\natexlab{b}})Shpitser and
  Pearl]{shpitser2006identification}
Shpitser, I. and Pearl, J.
\newblock Identification of joint interventional distributions in recursive
  semi-markovian causal models.
\newblock In \emph{Proceedings of the 21st National Conference on Artificial
  Intelligence}, volume~2, pp.\  1219--1226. AAAI Press, 2006{\natexlab{b}}.

\bibitem[Shpitser \& Pearl(2008)Shpitser and Pearl]{shpitser2008complete}
Shpitser, I. and Pearl, J.
\newblock Complete identification methods for the causal hierarchy.
\newblock \emph{Journal of Machine Learning Research}, 9\penalty0 (9), 2008.

\bibitem[Snell(1952)]{Snell1952TAMS}
Snell, J.~L.
\newblock Applications of martingale system theorems.
\newblock \emph{Transactions of the American Mathematical Society},
  73:\penalty0 293--312, 1952.

\bibitem[Srinivas et~al.(2012)Srinivas, Krause, Kakade, and
  Seeger]{srinivas2012information}
Srinivas, N., Krause, A., Kakade, S.~M., and Seeger, M.~W.
\newblock Information-theoretic regret bounds for gaussian process optimization
  in the bandit setting.
\newblock \emph{IEEE Transactions on Information Theory}, 58\penalty0
  (5):\penalty0 3250--3265, 2012.

\bibitem[Sussex et~al.(2022)Sussex, Makarova, and Krause]{model_based_cbo}
Sussex, S., Makarova, A., and Krause, A.
\newblock Model-based causal bayesian optimization, 2022.

\bibitem[Tian \& Pearl(2002)Tian and Pearl]{tian2002general}
Tian, J. and Pearl, J.
\newblock A general identification condition for causal effects.
\newblock In \emph{Eighteenth National Conference on Artificial Intelligence},
  volume~18, pp.\  567--573, 2002.

\bibitem[Tikka \& Karvanen(2018)Tikka and Karvanen]{tikka2018identifying}
Tikka, S. and Karvanen, J.
\newblock Identifying causal effects with the r package causaleffect.
\newblock \emph{arXiv preprint arXiv:1806.07161}, 2018.

\bibitem[Tikka et~al.(2021)Tikka, Hyttinen, and Karvanen]{tikka2019causal}
Tikka, S., Hyttinen, A., and Karvanen, J.
\newblock Causal effect identification from multiple incomplete data sources: A
  general search-based approach.
\newblock \emph{Journal of Statistical Software}, 99\penalty0 (5):\penalty0
  1--40, 2021.

\bibitem[Varici et~al.(2022)Varici, Shanmugam, Sattigeri, and
  Tajer]{varici2022causal}
Varici, B., Shanmugam, K., Sattigeri, P., and Tajer, A.
\newblock Causal bandits for linear structural equation models.
\newblock \emph{arXiv preprint arXiv:2208.12764}, 2022.

\bibitem[Vives(1990)]{submodularity_integration}
Vives, X.
\newblock Nash equilibrium with strategic complementarities.
\newblock \emph{Journal of Mathematical Economics}, 19\penalty0 (3):\penalty0
  305--321, 1990.
\newblock ISSN 0304-4068.
\newblock \doi{https://doi.org/10.1016/0304-4068(90)90005-T}.
\newblock URL
  \url{https://www.sciencedirect.com/science/article/pii/030440689090005T}.

\bibitem[Wald(1947)]{wald}
Wald, A.
\newblock \emph{Sequential Analysis}.
\newblock Wiley and Sons, New York, 1947.

\bibitem[Wu \& Frazier(2019)Wu and Frazier]{nomyopic_bo_1}
Wu, J. and Frazier, P.
\newblock Practical two-step lookahead bayesian optimization.
\newblock In Wallach, H., Larochelle, H., Beygelzimer, A., d~Alch\'{e}-Buc, F.,
  Fox, E., and Garnett, R. (eds.), \emph{Advances in Neural Information
  Processing Systems}, volume~32. Curran Associates, Inc., 2019.

\bibitem[Yue \& Kontar(2020)Yue and Kontar]{nomyopic_bo_5}
Yue, X. and Kontar, R.~A.
\newblock Why non-myopic bayesian optimization is promising and how far should
  we look-ahead? a study via rollout.
\newblock In Chiappa, S. and Calandra, R. (eds.), \emph{Proceedings of the
  Twenty Third International Conference on Artificial Intelligence and
  Statistics}, volume 108 of \emph{Proceedings of Machine Learning Research},
  pp.\  2808--2818. PMLR, 26--28 Aug 2020.

\end{thebibliography}

\appendix
\clearpage
\tableofcontents
\begin{appendices}
\label{sec:appendix}

\section{Notation}\label{sec:notations}
Random variables are denoted by upper-case letters (e.g. $X$) and their values by lower-case letters (e.g. $x$). The probability mass or density of a random variable $X$ is denoted by $\myP{X}$. We use $x \sim \myP{X}$ to denote that $x$ was sampled from $\myP{X}$. The expectation of a function $f$ with respect to a random variable $X$ is denoted by $\mathbb{E}_X[f]$. Sets of variables and their values are noted by bold upper-case and lower-case letters respectively (e.g. $\x$ and $\X$). Operators, function spaces and tuples are represented with upper case calligraphic letters (e.g. $\mathcal{M}$). The power set of a set $\X$ is denoted with $\mathcal{P}(\X)$. The set of all probability distributions over a set $\X$ (i.e. the $(|\X|-1)$-dimensional unit simplex) is denoted with $\Delta(\X)$. We make extensive use of the do-calculus (for details see \citep[\S 3.4]{pearl2000causality}). The domain of a variable is denoted by $\text{dom}(\cdot)$ where e.g. $x \in \text{dom}(X)$ and $\mat{x} \in \text{dom}(\mat{X}) \equiv x_1 \times x_2 \times  \ldots \times x_{|\mat{x}|} \in \text{dom}(X_1) \times \text{dom}(X_2) \times \ldots \times \text{dom}(X_{|\mat{x}|})$. The set of real numbers and the set of $n$-dimensional real vectors are denoted with $\mathbb{R}$ and $\mathbb{R}^{n}$ respectively. We adopt family relationships $\mathrm{pa(X)_{\graph}, ch(X)_{\graph}, an(X)_{\graph}}$ and $\mathrm{de(X)}_{\graph}$ to denote parents, children, ancestors and descendants of a given variable $X$ in a graph $\graph$; $\mathrm{Pa, Ch, An}$ and $\mathrm{De}$ extends $\mathrm{pa, ch, an}$ and $\mathrm{de}$ by including the argument as the result. For example $\Pa{X}{\graph} = \pa{X}{\graph} \cup \{ X \}$. With a set of variables as argument, $\pa{\X}{\graph} = \bigcup_{X \in \X} \pa{X}{\graph}$ and similarly defined for other relations.

\begin{table}[ht!]
  \centering
  \caption{Notation used throughout the paper.\label{table:notations}}
  \renewcommand{\arraystretch}{1.1} 
  \begin{tabular}{ll}
    \toprule
    Notation(s) & Description \\
    \midrule
    $\mathscr{M}$ & quadruple defining an \scm $\mathscr{M}=\scmdef$ \\
    $\V$ & set of endogeneous variables in an \scm \\
    $\U$ & set of exogeneous variables in an \scm \\
    $\Y$ & set of target variables in an \scm \\
    $\X$ & set of manipulative variables in an \scm \\
    $\N$ & set of non-manipulative variables in an \scm \\
    $\mat{F}$ & structural equation model of an \scm \\
    $\graph$ & causal diagram (\DAG) of an \scm (with nodes $\V$)\\
    $\E$ & edges of $\graph$ \\
    $\DO{\X}{\x}$ & the do-operator \cite{pearl2000causality} \\
    $\graph_{\overline{\X}}$ & the mutilated graph obtained by deleting from $\graph$ all arcs pointing to nodes in $\X$ \\
    $\mathbf{M}^{\mathbf{V}}_{\graph,Y}$ & set of \mis{}s (\cref{def:mis}) for an \scm with endogeneous variables $\V$ and target variable $Y$ \\
    $\mathbf{P}_{\graph,Y}^{\mathbf{V}}$ & set of \pomis{}s (\cref{def:pomis}) for an \scm with endogeneous variables $\V$ and target variable $Y$ \\
    $\mathbf{O}^{\mathbf{X}}_{\graph,Y}$ & \mos (\cref{def:mos}) for an identifiable intervention $\myP{Y \mid \DO{\X}{\x}}$ \\
    $Q_{Y}^{\X}$ & do-calculus expression to estimate $\myP{Y \mid \DO{\X}{\x}}$ from $P(\V)$ \\
    $\mu(\X^{'}, \x^{'})$ & shorthand for $\mathbb{E}[Y \mid \DO{\X'}{\x'}]$ \\
    $\X^{*}, \x^{*}$ & optimal intervention set $\X^{*}$ and intervention levels $\x^{*}$ \\
    $e$ & estimation procedure, $e \in \{\intervene, \observe\}$ \\
    $\intervene$ & estimation by intervention  \\
    $\observe$ & estimation by observation  \\
    $c(\X^{'},e)$ & cost of estimating $\mathbb{E}[Y \mid \DO{\X'}{\x'}]$ using procedure $e\in \{\intervene, \observe\}$ \\
    $T$ & time horizon of the optimisation \\
    $K$ & maximum evaluation cost \\
    $\D_t$ & dataset of measured observations and interventions at stage $t$ of the optimisation \\
    $\widehat{\mu}_{\D_t}$ & probabilistic model of $\mu$ based on $\D_t$ \\
    $\widehat{\mat{F}}_{\D_t}$ & probabilistic model of $\mat{F}$ based on $\D_t$ \\
    $\mathcal{M}_t$ & stopping problem at stage $t$ of the optimisation \\
    $\mathbf{S}_k,\mathbf{o}_k$ & state and observation at stage $k$ of an optimal stopping problem \\
    $\bot$ & terminal state of an optimal stopping problem \\
    $r(\mathbf{S}_k)$ & stopping reward at stage $k$ of an optimal stopping problem \\
    $\mathcal{T}$ & stopping time \\
    $\gamma$ & discount factor for an optimal stopping problem \\
    $\mathscr{S},\mathscr{C}$ & stopping and continuation sets for an optimal stopping problem\\
    \bottomrule
  \end{tabular}
\end{table}
\FloatBarrier

\section{Modelling assumptions}
\label{sec:assumptions}

This sections contains two tables which detail our modelling assumptions on the causal inference and optimal stopping sides respectively.

\begin{table}[ht!]
    \centering
    \caption{Causal inference assumptions.\label{table:ci_assump}}
    \begin{tabular}{ll}
    \toprule
    Assumption & Description \\
      \midrule
    $|\V| < \infty$ & the set of endogenous variables is finite \\
    univariate target $Y \in \V$ & we do not consider multivariate targets $\Y \subset \V$, $|\Y| > 1$ \\
    target $Y \in \V$ is bounded & $|y| \leq M < \infty$ for some $M \in \mathbb{R}$ and all $y \in \text{dom}(Y)$ \\
    atomic interventions & interventions $\DO{\X'}{\x'}$ in \labelcref{eq:cbo} are atomic \citep{correa2020calculus} \\
    stationary \DAG topology & the \DAG topology is independent of time $t$ \\
    stationary $\mathbf{F}$ & the functions $f_i \in \mat{F}$ are independent of time $t$\\
      $\graph$ is a \DAG & causal diagrams are free of cycles \\
    identifiability methods & methods based on $P(\V)$ only (i.e. non-experimental \citep[Table 1]{tikka2019causal}) \\
    \bottomrule
    \end{tabular}
\end{table}

\begin{table}[ht!]
    \centering
    \caption{Optimal stopping assumptions.\label{table:os_assump}}
    \begin{tabular}{ll}
    \toprule
    Assumption & Description \\
      \midrule
      $c(\mathbf{L},\observe) > 0 \text{ }\forall \mathbf{L} \in \mathcal{P}(\V)$ & positive observation costs\\
    $c(\mathbf{M},\intervene) > 0 \text{ }\forall \mathbf{M} \in \mathcal{P}(\V)$ & positive intervention costs\\
      $c(\X^{'},\observe) < M < \infty \text{ } \forall \mathbf{L} \in \mathcal{P}(\V)$ & bounded observation costs ($M \in \mathbb{R}$)\\
      $c(\X^{'},\intervene) < M < \infty \text{ } \forall \mathbf{M} \in \mathcal{P}(\X)$ & bounded intervention costs ($M \in \mathbb{R}$)\\
      $T<\infty$ & finite time horizon $\implies $ $\mathcal{T}^{*} < \infty$ \labelcref{eq:stopping_time_problem}\\
    \bottomrule
    \end{tabular}
\end{table}

\section{Causal-effect derivation}
\label{sec:causal-effect derivation}

Derivation \citep{bareinboim2016causal} for estimating the causal effect of $X$ on $Y$ in \cref{fig:admg_a}. We use the shorthand $P_X(Y)$ here to denote the interventional distribution $\myP{Y \mid \DO{X}{x}}$.
See \citep[Thm. 3.4.1]{pearl2000causality} for the three rules of do-calculus.
\begin{proof}
	Summing over $Z$ gives
	\begin{equation}
		P_{X}\left(Y\right) = \sum_{Z}{P_{X}\left(Y,Z\right)}
	\end{equation}
	By C-component factorisation, we have
	\begin{equation}
		= \sum_{Z}{P_{X,Y}\left(Z\right)P_{X,Z}\left(Y\right)} \label{eq6}
	\end{equation}
	\textbf{Task 1: Compute $P_{X,Y}\left(Z\right)$} \\\\
	The third rule of do-calculus can be applied using the independence $\left(Y \perp Z | X\right)_{\graph_{\overline{X,Y}}}$ (see \cref{fig:causal_graph_example}).
	\begin{equation}
		P_{X,Y}\left(Z\right)
		= {{P_{X}\left(Z\right)}}
	\end{equation}
	The second rule of do-calculus can be applied using the independence $\left(X \perp Z\right)_{\graph_{\overline{}\underline{X}}}$ (see \cref{fig:causal_graph_example}).
	\begin{equation}
		= {{P_{}\left(Z \middle| X\right)}} \label{eq9}
	\end{equation}
	\textbf{Task 2: Compute $P_{X,Z}\left(Y\right)$} \\\\
	The third rule of do-calculus can be applied using the independence $\left(X \perp Y | Z\right)_{\graph_{\overline{X,Z}}}$ (see \cref{fig:causal_graph_example}).
	\begin{equation}
		P_{X,Z}\left(Y\right)
		= P_{Z}\left(Y\right)
	\end{equation}
	Summing over $X'$ gives
	\begin{equation}
		= \sum_{X'}{P_{Z}\left(X',Y\right)}
	\end{equation}
	\textbf{Task 2.1: Compute $P_{Z}\left(X',Y\right)$} \\\\
	We compute the effect inside the sum.
	By the chain rule, we have
	\begin{equation}
		P_{Z}\left(X',Y\right)
		= P_{Z}\left(Y \middle| X'\right)P_{Z}\left(X'\right)
	\end{equation}
	The third rule of do-calculus can be applied using the independence $\left(Z \perp X'\right)_{\graph_{\overline{Z}}}$ (see \cref{fig:causal_graph_example}).
	\begin{equation}
		= P_{Z}\left(Y \middle| X'\right){{P_{}\left(X'\right)}}
	\end{equation}
	The second rule of do-calculus can be applied using the independence $\left(Z \perp Y | X'\right)_{\graph_{\overline{}\underline{Z}}}$ (refer to \cref{fig:causal_graph_example}).
	\begin{equation}
		= {{P_{}\left(Y \middle| X',Z\right)}}P_{}\left(X'\right) \label{eq16}
	\end{equation}
	Substituting \labelcref{eq9} and \labelcref{eq16} back into \labelcref{eq6}, we get
	\begin{flalign}
		P_{X}\left(Y\right) = \sum_{Z}{P\left(Z \middle| X\right)\sum_{X'}{P\left(Y \middle| X',Z\right)P\left(X'\right)}} \label{eq17}
	\end{flalign}
      \end{proof}

\section{Proof of Theorem \ref{thm:closed_stopping_set}}\label{sec:thm_1_proof}
\begin{proof}
For ease of notation, let $\mu_{\mathbf{S}_k}$, $c_{\observe}$, and $V_{\mathbf{S}_k}$ be a shorthands for $\mu_{\mathbf{S}_k}(\X_k',\x_{k}')$, $c(\X'_t, \observe)$, and $\frac{\text{Vol}(\V)}{\text{Vol}(\mathbf{S}_k)}$, respectively.
\begin{align}
&\mathbf{S}_k \in \mathscr{S}_1\\
&\implies r(\mathbf{S}_{k}) \geq \mathbb{E}_{\mathbf{o}_{k+1}}\left[r(\mathbf{S}_{k} \cup \{\mathbf{o}_{k+1}\})\right]-c_{\observe}\label{eq:thm1_step1}\\
  &\iff \eta I(\mathbf{S}_k;\mu) -\kappa\widehat{\mu}_{\mathbf{S}_k} - \tau V_{\mathbf{S}_k} \geq \mathbb{E}_{\mathbf{o}_{k+1}}\left[\eta I(\mathbf{S}_{k+1};\mu) -\kappa\widehat{\mu}_{\mathbf{S}_{k+1}}-\tau V_{\mathbf{S}_{k+1}}\right]-c_{\observe}\label{eq:thm1_step2}\\
&\implies \eta I(\mathbf{S}_k;\mu) -\kappa\widehat{\mu}_{\mathbf{S}_k} \geq \mathbb{E}_{\mathbf{o}_{k+1}}\left[\eta I(\mathbf{S}_{k+1};\mu) -\kappa\widehat{\mu}_{\mathbf{S}_{k+1}}\right]-c_{\observe}\label{eq:thm1_step22}\\
  &\implies \kappa\left(\mathbb{E}_{\mathbf{o}_{k+1}}\left[\widehat{\mu}_{\mathbf{S}_{k+1}}\right]-\widehat{\mu}_{\mathbf{S}_k}\right) \geq \eta\left(\mathbb{E}_{\mathbf{o}_{k+1}}\left[I(\mathbf{S}_{k+1};\mu)\right]-I(\mathbf{S}_k;\mu)\right)-c_{\observe}\label{eq:thm1_step3_0}\\
&\implies \kappa\left(\mathbb{E}_{\mathbf{o}_{k+2}}\left[\widehat{\mu}_{\mathbf{S}_{k+1}}\right]-\widehat{\mu}_{\mathbf{S}_k}\right) \geq \eta\left(\mathbb{E}_{\mathbf{o}_{k+1}}\left[I(\mathbf{S}_{k+1};\mu)\right]-I(\mathbf{S}_k;\mu)\right)-c_{\observe}\label{eq:thm1_step3}\\
  &\implies \kappa\left(\mathbb{E}_{\mathbf{o}_{k+2}}\left[\widehat{\mu}_{\mathbf{S}_{k+2}}\right]-\widehat{\mu}_{\mathbf{S}_{k+1}}\right) \geq \eta\left(\mathbb{E}_{\mathbf{o}_{k+1}}\left[I(\mathbf{S}_{k+1};\mu)\right]-I(\mathbf{S}_k;\mu)\right)-c_{\observe}\label{eq:thm1_step4}\\
&\implies \kappa\left(\mathbb{E}_{\mathbf{o}_{k+2}}\left[\widehat{\mu}_{\mathbf{S}_{k+2}}\right]-\widehat{\mu}_{\mathbf{S}_{k+1}}\right) \geq \eta\left(\mathbb{E}_{\mathbf{o}_{k+2}}\left[I(\mathbf{S}_{k+1};\mu)\right]-I(\mathbf{S}_k;\mu)\right)-c_{\observe}\label{eq:thm1_step4_0}\\
&\implies \kappa\left(\mathbb{E}_{\mathbf{o}_{k+2}}\left[\widehat{\mu}_{\mathbf{S}_{k+2}}\right]-\widehat{\mu}_{\mathbf{S}_{k+1}}\right) \geq \eta\left(\mathbb{E}_{\mathbf{o}_{k+2}}\left[I(\mathbf{S}_{k+2};\mu)\right]-I(\mathbf{S}_{k+1};\mu)\right)-c_{\observe}\label{eq:thm1_step5}\\
&\implies \mathbf{S}_{k+1} \in \mathscr{S}_1\label{eq:thm1_step6}
\end{align}
\Cref{eq:thm1_step1} follows from the Bellman equation \labelcref{eq:stopping_time_bellman} and \labelcref{eq:thm1_step2} follows because $c(\pi_{\mathrm{O}}(\mathbf{S}), \intervene)$ is non-increasing in $|\mathbf{S}|$, which is implied by assumption a). Equation \labelcref{eq:thm1_step22} holds because $\frac{\text{Vol}(\V)}{\text{Vol}(\mathbf{S}_k)} \geq \frac{\text{Vol}(\V)}{\text{Vol}(\mathbf{S}_{k+1})}$ by definition. Equation (\ref{eq:thm1_step3_0}) follows from linearity of $\mathbb{E}$. Equation (\ref{eq:thm1_step3}) follows from stationarity of the observation distribution $\widehat{P}_{\mathbf{S}_t}$. \Cref{eq:thm1_step4} follows from supermodularity of $\widehat{\mu}_{\mathbf{S}}$ (which is preserved by integration \cite{submodularity_integration}), i.e. $\widehat{\mu}_{\mathbf{S}_{k}\cup {\mathbf{o}_{k+2}}} - \widehat{\mu}_{\mathbf{S}_{k}} \leq \widehat{\mu}_{\mathbf{S}_{k+1}\cup {\mathbf{o}_{k+2}}} - \widehat{\mu}_{\mathbf{S}_{k+1}}$, which is implied by assumption a) and $\mathbf{S}_{k} \subset \mathbf{S}_{k+1}$. Similarly, \labelcref{eq:thm1_step4_0} and \labelcref{eq:thm1_step5} follows from stationarity of $\widehat{P}_{\mathbf{S}_t}$ and submodularity of $I$, respectively. More specifically, \labelcref{eq:thm1_step5} holds because $\mathbf{S}_{k} \subset \mathbf{S}_{k+1}$ and $I(\mathbf{S}_{k} \cup \{\mathbf{o}_{k+2}\};\mu) - I(\mathbf{S}_{k};\mu) \geq I(\mathbf{S}_{k+1} \cup \{\mathbf{o}_{k+2}\};\mu) - I(\mathbf{S}_{k+1};\mu)$, which is implied by assumption b). Finally, \labelcref{eq:thm1_step6} follows from \labelcref{eq:stopping_time_bellman}.
\end{proof}
\section{Proof of Corollary  \ref{cr:olsa}}\label{sec:cr_1_proof}
\begin{proof}
The corollary follows from a well-known result in optimal stopping theory (see \cite[pp. 147-149]{bert05} and \cite[Prop. 7.3]{bather_decision_theory}). The rule in \labelcref{eq:olsa} is known as a \textit{one-step lookahead policy}. To prove \labelcref{eq:olsa} we note that, by definition of $\mathscr{S}_1$, \labelcref{eq:olsa} holds if $k=T-1$. Assume by induction that \labelcref{eq:olsa} holds for some integer $n+1 < T-1$ and consider a state $\mathbf{S}_{n} \in \mathscr{S}_1$. Then
\begin{align}
  \mathcal{T}^{*}=n \iff r(\mathbf{S}_{n}) &\geq \mathbb{E}_{\mathbf{o}_{n+1}}[V^{*}(\mathbf{S}_{n+1})]-c(\X'_t, \observe)\label{eq:corollary_1_1}\\
  &=\mathbb{E}_{\mathbf{o}_{n+1}}[r(\mathbf{S}_{n+1})]-c(\X'_t, \observe)\label{eq:corollary_1_2}\\
  &\leq r(\mathbf{S}_{n})\label{eq:corollary_1_3}\\
  \implies \mathbf{S}_{n} &\in \mathscr{S}_{n}\label{eq:corollary_1_4}
\end{align}
where \labelcref{eq:corollary_1_1} follows from the Bellman equation (\ref{eq:stopping_time_bellman}) and \labelcref{eq:corollary_1_2} follows from Theorem \ref{thm:closed_stopping_set} and the fact that $\mathscr{S}_1$ is closed. We then have, by induction and by \labelcref{eq:corollary_1_4}, that $\mathbf{S}' \in \mathscr{S}_1 \implies \mathbf{S}' \in \mathbf{S}_l$, for all $l \in \{1,\hdots,T-1\}$. Then, since $\mathscr{S}_l \subseteq \mathscr{S}_{l-1} \subseteq \hdots \subseteq \mathscr{S}_{1}$ by definition, we have that $\mathscr{S}_1=\mathscr{S}_2=\hdots=\mathscr{S}_{T-1}$, which directly implies \labelcref{eq:olsa}.
\end{proof}

\section{Additional evaluation results}\label{sec:ablation}
This appendix contains additional evaluation results, complementing those in the main body of the paper. \Cref{sec:extra_chain} contains results for the chain \scm (see \cref{fig:toygraph}); \cref{sec:extra_chain_uc} contains results for the chain \scm with an unobserved confounder (see \cref{fig:toygraph_w_uc}); \cref{sec:extra_psa} contains results for the PSA \scm (see \cref{fig:psa_level});  and \cref{sec:extra_synthetic} contains results for the synthetic \scm with causal graph in \cref{fig:causal_graph_example_two}.

To re-emphasise, in all experiments, we only explore the \pomis{}s for each \scm{}  -- for a complete pseudo-algorithm see \cref{alg:os_pseudocode}.

\subsection{Chain \scm}\label{sec:extra_chain}
The chain \scm (see \cref{fig:toygraph}) is a synthetic \scm that is benchmarked in both \cite{cbo} and \cite{model_based_cbo}. \Cref{fig:toygraph_gps} and \cref{fig:toygraph_gps_extra} show the estimated probabilistic models $\widehat{\mu}$ and $\widehat{\mat{F}}$ when running \cbo with four different policies for balancing the intervention-observation trade-off: \textit{i}) the $\epsilon$-greedy policy used in \cite{cbo}; \textit{ii}) the \osco approach described in \cref{sec:os_formulation}; \textit{iii}) the \textsc{observe} baseline (which always observes); and \textit{iv}) the \textsc{random} baseline, which selects between intervening and observing uniformly at random.

\begin{figure}[!htb]
  \centering
\hspace{-1.5cm}
\begin{subfigure}[t]{0.45\textwidth}
  \centering
  \scalebox{1.2}{
    \includegraphics[width=0.95\linewidth]{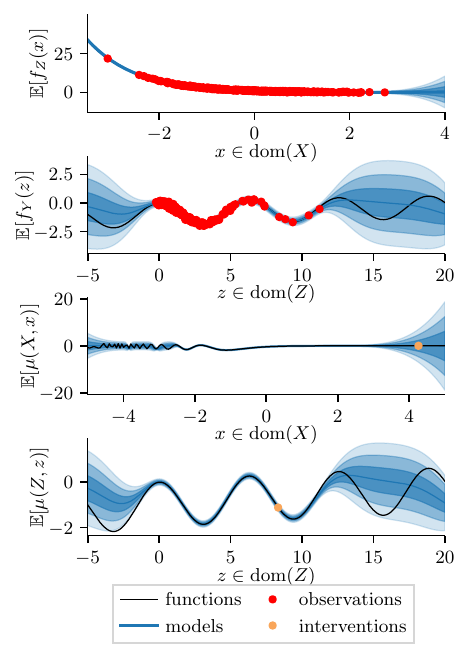}
}
    \caption{\cbo \textsc{observe}.}
     \label{fig:cbo_gps_observe_toygraph}
\end{subfigure}
\hspace{7mm}%
\begin{subfigure}[t]{0.45\textwidth}
  \centering
  \scalebox{1.2}{
    \includegraphics[width=0.95\linewidth]{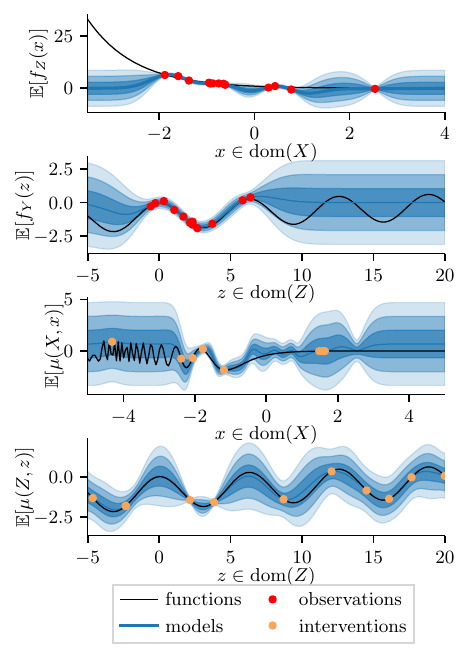}
}
    \caption{\cbo \textsc{random}.}
     \label{fig:cbo_gps_random_toygraph}
   \end{subfigure}
 \caption{Collected data and estimated models from running \cbo to solve \labelcref{eq:cbo} for the chain \scm \cite[Fig. 3]{cbo} (see \cref{fig:toygraph}); the left plot (\cref{fig:cbo_gps_observe_toygraph}) shows results when using a policy that always observes and the right plot (\cref{fig:cbo_gps_random_toygraph}) shows results when using a policy that selects randomly between observing and intervening; the blue curves and the shaded blue areas show the mean and standard deviation of the estimated models $\widehat{\mat{F}}$ and $\widehat{\mu}$; the red and orange dots show observations and interventions; the black lines show the \scm functions $\mat{F}$ and the causal effects $\mu$.}
 \label{fig:toygraph_gps_extra}
\end{figure}

We note in the lowest plots that the optimal intervention is $\DO{Z}{-3.20}$ (with target value $Y=-2.17$) and that this intervention is found in all cases except for the \textsc{observe} baseline (\cref{fig:cbo_gps_observe_toygraph}). That the \textsc{observe} baseline does not find the optimal intervention is expected as the probability of observing the optimal configuration without intervening is low. We further note that all policies that collect observations are able to accurately estimate the interventional distributions through the do-calculus (see e.g. \cref{fig:cbo_gps_observe_toygraph} and \cref{fig:cbo_os_gps_toygraph}). The main differences between \cbo with $\epsilon$-greedy and \cbo with \osco (see \cref{fig:toygraph_gps}) are a) \cbo with $\epsilon$-greedy collects only $3$ observations, spending most of the evaluation budget on interventions, whereas \cbo with \osco uses most of the evaluation budget to collect observations; and b) that \cbo with $\epsilon$-greedy observes all endogeneous variables whereas  \cbo with \osco only observes the \mos $\mathbf{O}^{Z}_{\graph,Y}=\{Z,Y\}$. That \cbo uses most of the evaluation budget on intervening whereas \cbo with \osco uses most of the budget on observations can be explained by two main reasons. First, the definition of $\epsilon$ in \cite[Eq. 6]{cbo} implies that the probability of observing in \cbo with $\epsilon$-greedy is close to $0$ when the number of previously collected observations is low. Second, the optimal stopping formulation in \labelcref{eq:olsa} implies that \cbo with \osco will observe rather than intervene when it is more cost-effective.

\begin{figure}
  \centering
  \scalebox{1}{
    \includegraphics[width=1\linewidth]{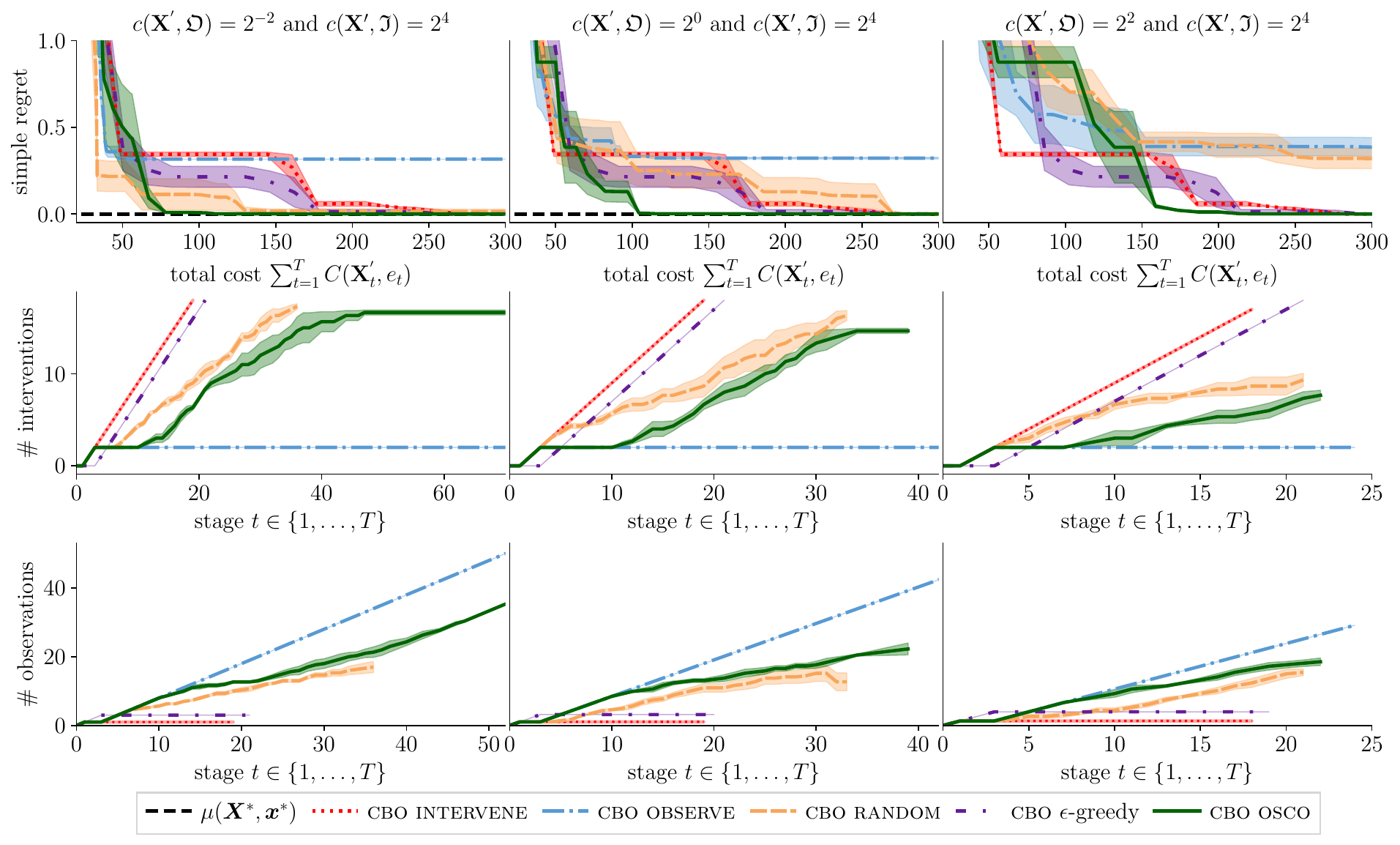}
}
    \caption{Convergence curves for \cbo with different policies for balancing the intervention-observation trade-off when solving \labelcref{eq:cbo} for the chain \scm (\cref{fig:toygraph}); the columns relate to evaluations with different observation costs $c(\X^{'}, \mathfrak{O})$; the top row shows the objective value against the evaluation cost; the middle and bottom rows show the number of interventions ($e_t=\mathfrak{I}$) and observations ($e_t=\mathfrak{O}$) at different stages of the optimisation; the curves indicate the mean and standard deviation ($\pm \frac{\sigma}{\sqrt[]{3}}$) over three evaluations with different random seeds.}
     \label{fig:toygraph_convergence}
\end{figure}

\Cref{fig:toygraph_convergence} shows convergence curves of \cbo with \osco and the baselines introduced above for different observation costs $c(\X, \observe)$. We note that the \textsc{observe} and \textsc{random} baselines do not always converge to the optimum within the prescribed evaluation budget (see \cref{sec:hyperparameters} for the list of hyperparameters). We further observe that \cbo with \osco on average reaches the optimum at a lower cost than all baselines.

We also compare the performance of \osco on the chain \scm with the performance of \mcbo \cite{model_based_cbo}. \Cref{fig:toygraph_mcbo_convergence} shows convergence curves of \mcbo with \osco and the baselines introduced in \cref{sec:experiments} for different observation costs $c(\X, \observe)$. We note that none of the \mcbo policies find the optimum. This is consistent with the findings in \cite[Fig. 6]{model_based_cbo} and can be explained by the design of \mcbo, which is optimised for minimising the cumulative regret rather than than the simple regret. We further note that the measured benefit of adding \osco to \mcbo for the chain \scm is lower than that of \cbo (cf. \cref{fig:toygraph_mcbo_convergence}). More specifically, when the observation cost is low (see e.g. the right plot in \cref{fig:toygraph_mcbo_convergence}), \mcbo with \osco yields better results than plain \mcbo. When the observation costs are high however (see e.g. the left plot in \cref{fig:toygraph_mcbo_convergence}), \mcbo with \osco leads to slower convergence. One reason why \osco works better with \cbo than \mcbo is that it can utilise \mos (\cref{def:mos}) to limit the number of observations (see \cref{sec:experiments} for details). We speculate that another reason is the different way of integrating observational data in the probabilistic models $\widehat{\mu}$ and $\widehat{\mat{F}}$. In \cbo, the observational data is integrated with the interventional data by using a causal prior on the interventional distributions whereas in \mcbo the functions $\widehat{\mat{F}}$ are fitted directly based on both observational and interventional data.

\begin{figure}
  \centering
  \scalebox{1}{
    \includegraphics[width=1\linewidth]{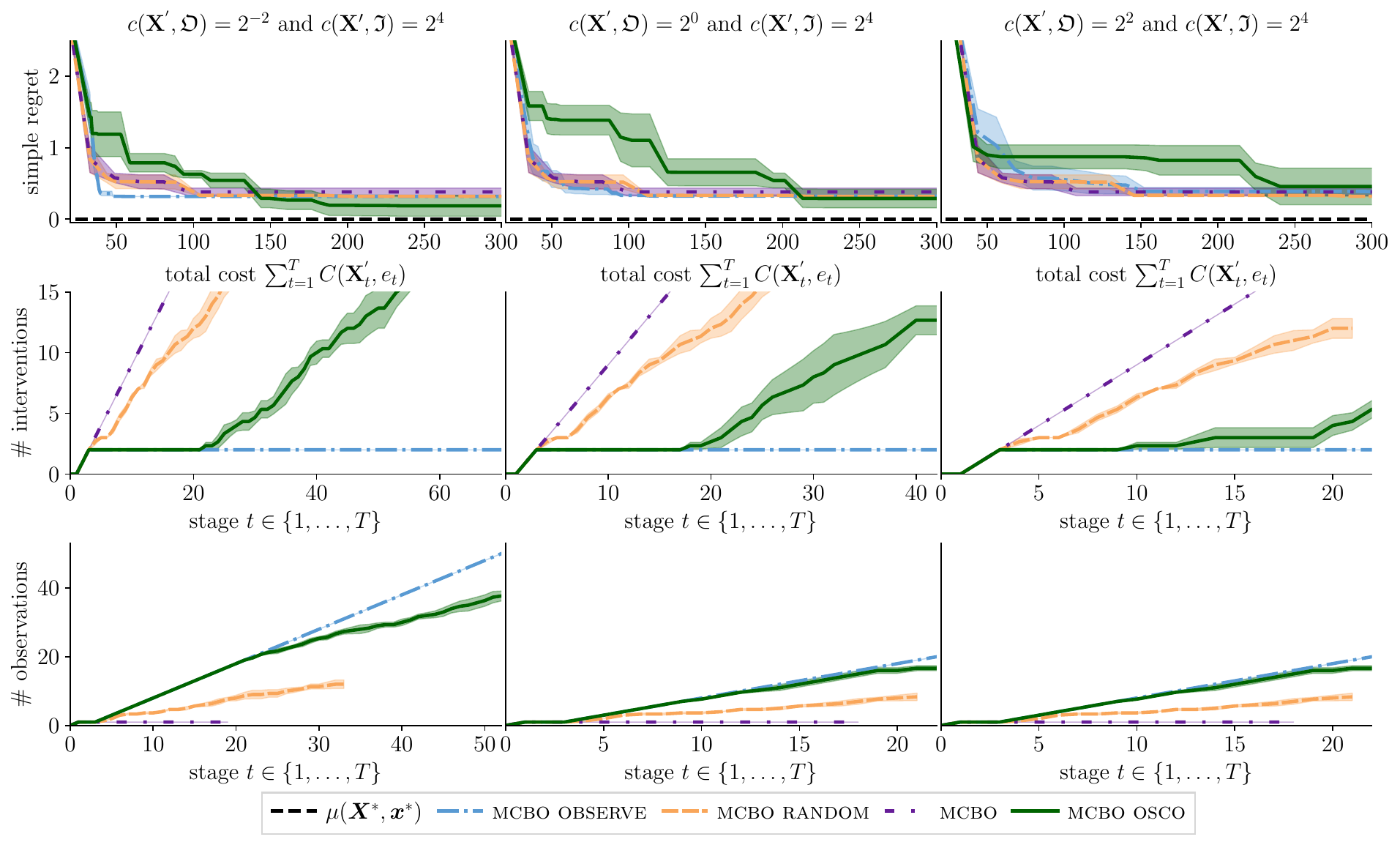}
}
    \caption{Convergence curves for \mcbo with different policies for balancing the intervention-observation trade-off when solving \labelcref{eq:cbo} for the chain \scm (\cref{fig:toygraph_gps}); the columns relate to evaluations with different observation costs $c(\X^{'}, \mathfrak{O})$; the top row shows the objective value against the evaluation cost; the middle and bottom rows show the number of interventions ($e_t=\mathfrak{I}$) and observations ($e_t=\mathfrak{O}$) at different stages of the optimisation; the curves indicate the mean and standard deviation ($\pm \frac{\sigma}{\sqrt[]{3}}$) over three evaluations with different random seeds.}
     \label{fig:toygraph_mcbo_convergence}
\end{figure}




\subsection{Chain \scm with an unobserved confounder}\label{sec:extra_chain_uc}
\Cref{fig:chain_uc_convergence} and \cref{fig:chain_uc_convergence_mcbo} show convergence curves of \cbo and \mcbo with different policies for balancing the intervention-observation trade-off for the chain \scm with an unobserved confounder (see \cref{fig:toygraph_w_uc}). We observe that \cbo with \osco performs best on average and that the results resemble those obtained for the chain \scm without the unobserved confounder (see \cref{sec:extra_chain}). Looking at the second and third rows in the figures, we see that \cbo and \mcbo with \osco focuses on collecting observations in the beginning of the optimisation and then successively increases the frequency of interventions. This contrasts with \cbo and \mcbo without \osco, which almost exclusively intervenes.

\begin{figure}
  \centering
  \scalebox{1}{
    \includegraphics[width=1\linewidth]{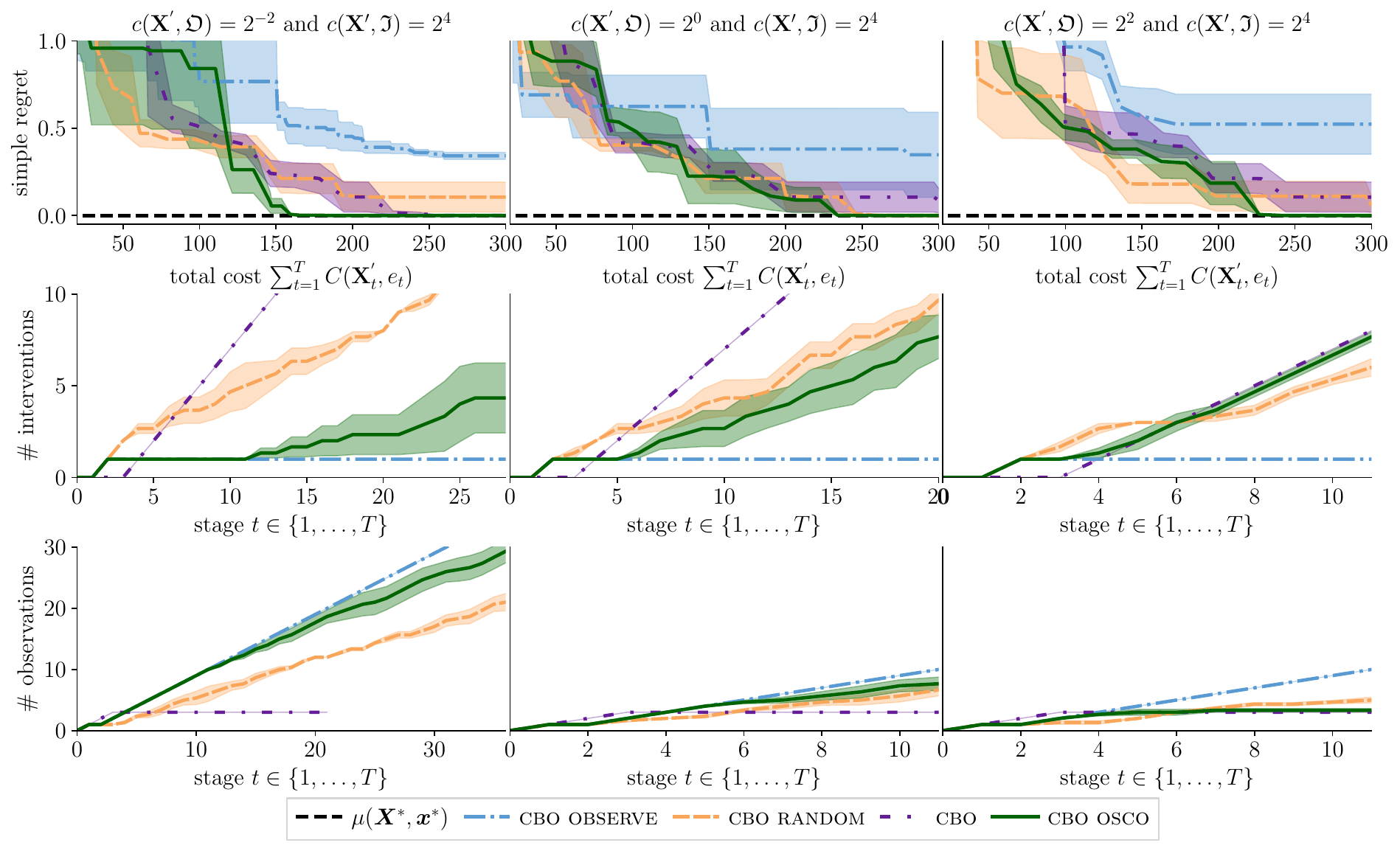}
}
    \caption{Convergence curves for \cbo with different policies for balancing the intervention-observation trade-off when solving \labelcref{eq:cbo} for the chain \scm with an unobserved confounder (\cref{fig:admg_a}); the columns relate to evaluations with different observation costs $c(\X^{'}, \mathfrak{O})$; the top row shows the objective value against the evaluation cost; the middle and bottom rows show the number of interventions ($e_t=\mathfrak{I}$) and observations ($e_t=\mathfrak{O}$) at different stages of the optimisation; the curves indicate the mean and standard deviation ($\pm \frac{\sigma}{\sqrt[]{3}}$) over three evaluations with different random seeds.}
     \label{fig:chain_uc_convergence}
\end{figure}

\begin{figure}
  \centering
  \scalebox{1}{
    \includegraphics[width=1\linewidth]{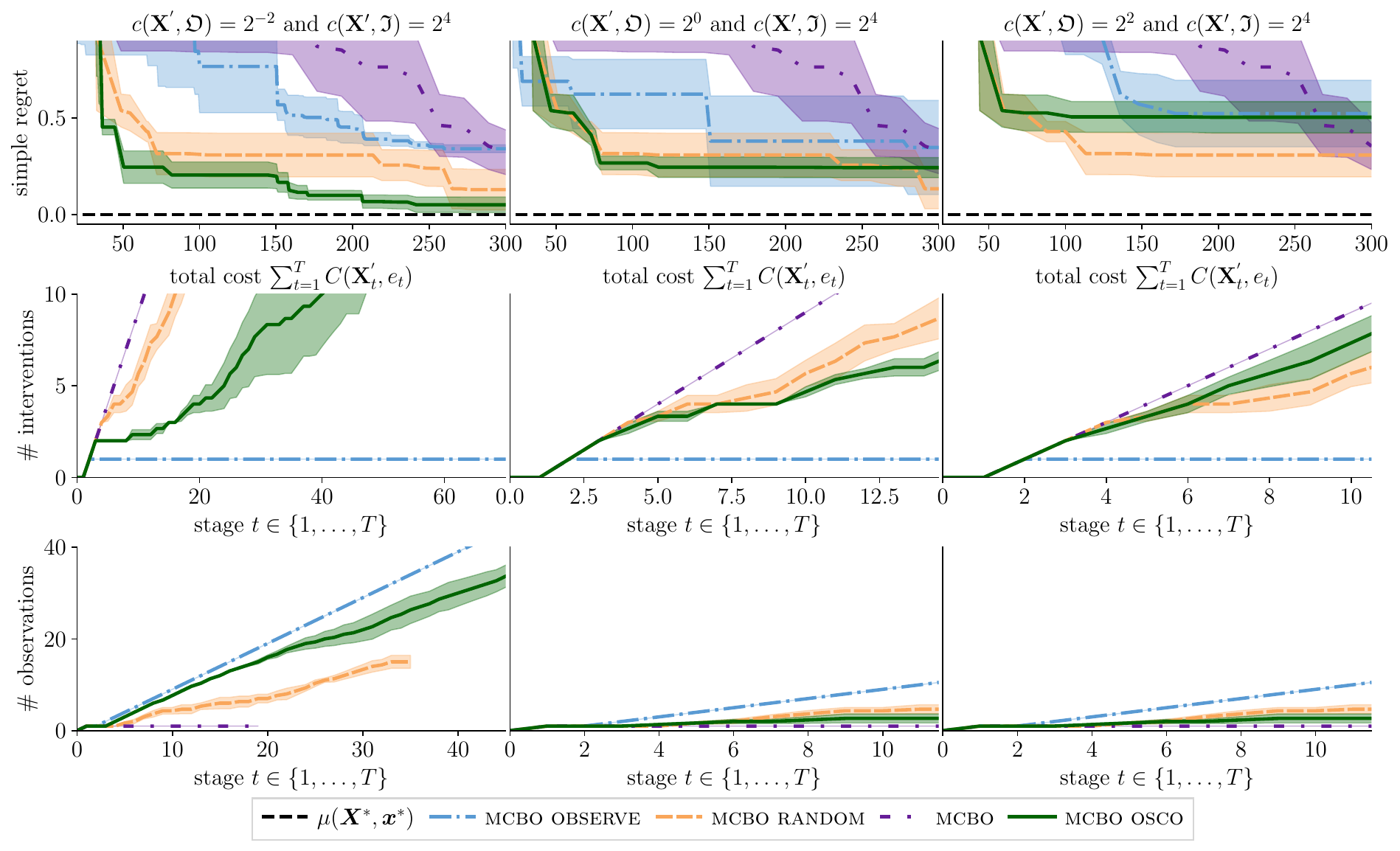}
}
    \caption{Convergence curves for \mcbo with different policies for balancing the intervention-observation trade-off when solving \labelcref{eq:cbo} for the chain \scm with an unobserved confounder (\cref{fig:admg_a}); the columns relate to evaluations with different observation costs $c(\X^{'}, \mathfrak{O})$; the top row shows the objective value against the evaluation cost; the middle and bottom rows show the number of interventions ($e_t=\mathfrak{I}$) and observations ($e_t=\mathfrak{O}$) at different stages of the optimisation; the curves indicate the mean and standard deviation ($\pm \frac{\sigma}{\sqrt[]{3}}$) over three evaluations with different random seeds.}
     \label{fig:chain_uc_convergence_mcbo}
\end{figure}

\subsection{PSA \scm}\label{sec:extra_psa}
The PSA \scm (see \cref{fig:psa_level}) is based on a real healthcare setting \cite{ferro2015use} where interventions correspond to dosage prescriptions of statins and/or aspirin to control Prostate-Specific Antingen (PSA) levels, which should be minimised. This \scm is benchmarked in both \cite{cbo} and \cite{model_based_cbo}.

\begin{figure}
  \centering
  \scalebox{1}{
    \includegraphics[width=1\linewidth]{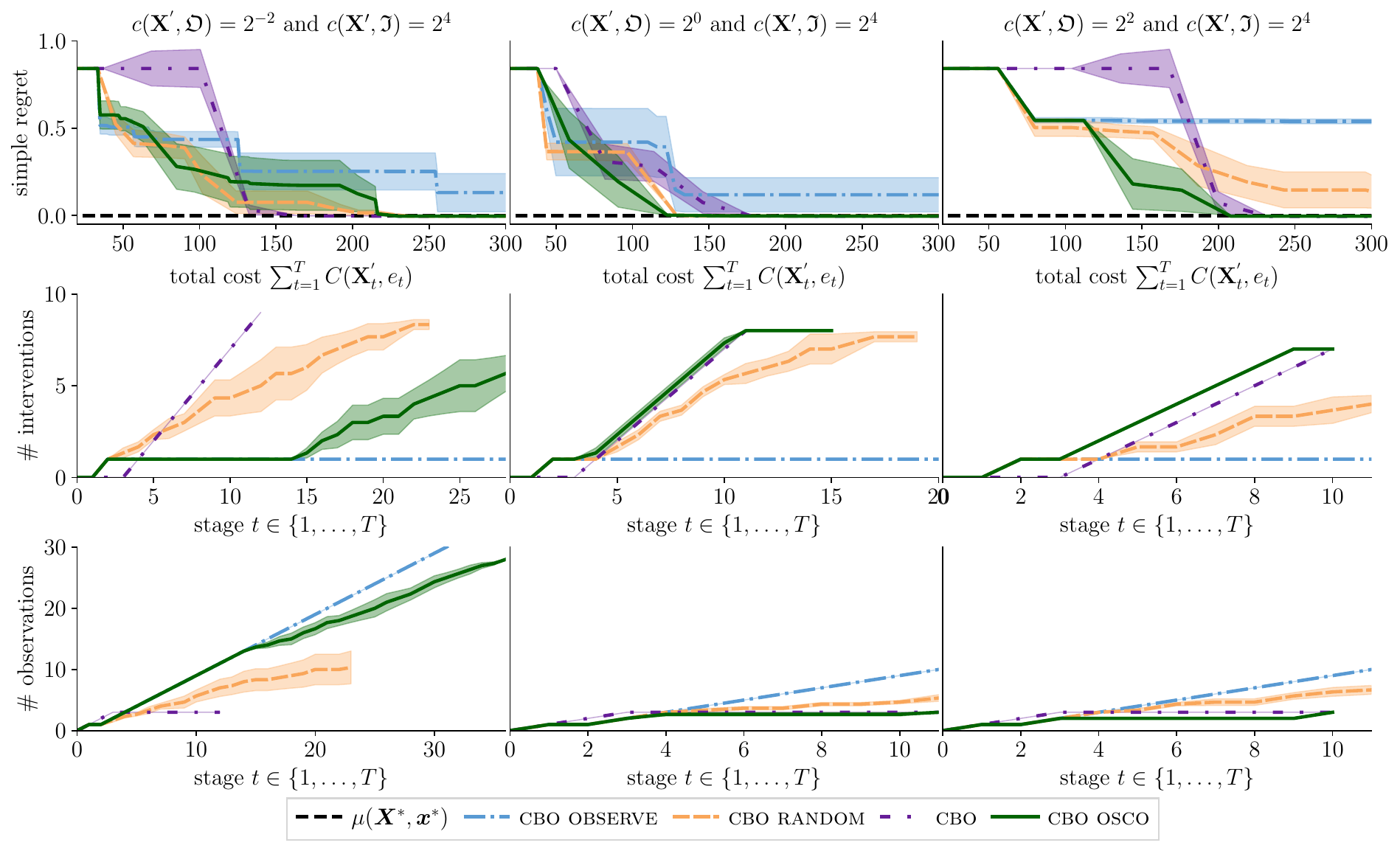}
}
    \caption{Convergence curves for \cbo with different policies for balancing the intervention-observation trade-off when solving \labelcref{eq:cbo} for the PSA \scm (\cref{fig:psa_level}); the columns relate to evaluations with different observation costs $c(\X^{'}, \mathfrak{O})$; the top row shows the objective value against the evaluation cost; the middle and bottom rows show the number of interventions ($e_t=\mathfrak{I}$) and observations ($e_t=\mathfrak{O}$) at different stages of the optimisation; the curves indicate the mean and standard deviation ($\pm \frac{\sigma}{\sqrt[]{3}}$) over three evaluations with different random seeds.}
     \label{fig:psa_convergence}
\end{figure}

\Cref{fig:psa_convergence} and \cref{fig:psa_mcbo_convergence} show convergence curves of \cbo and \mcbo with \osco and the baselines introduced in \cref{sec:experiments} for different observation costs $c(\X, \observe)$. We note that the only policies that consistently find the optimum are \cbo, \cbo with \osco, \mcbo, and \mcbo with \osco. We also note that \osco performs, on average, better than \cbo and \mcbo. The differences are however relatively small. We believe that the reason why the differences are relatively small is that both \cbo and \mcbo finds the optimum in the chain \scm after only a couple of interventions, diminishing the need to utilise observational data. This is because the optimal intervention in the PSA \scm is at the endpoints of the domains, i.e. $(\X^{*},\x^{*}) = (\{\text{aspirin}, \text{statin}\}, (0,1))$, which is easy to find.

\begin{figure}
  \centering
  \scalebox{1}{
    \includegraphics[width=1\linewidth]{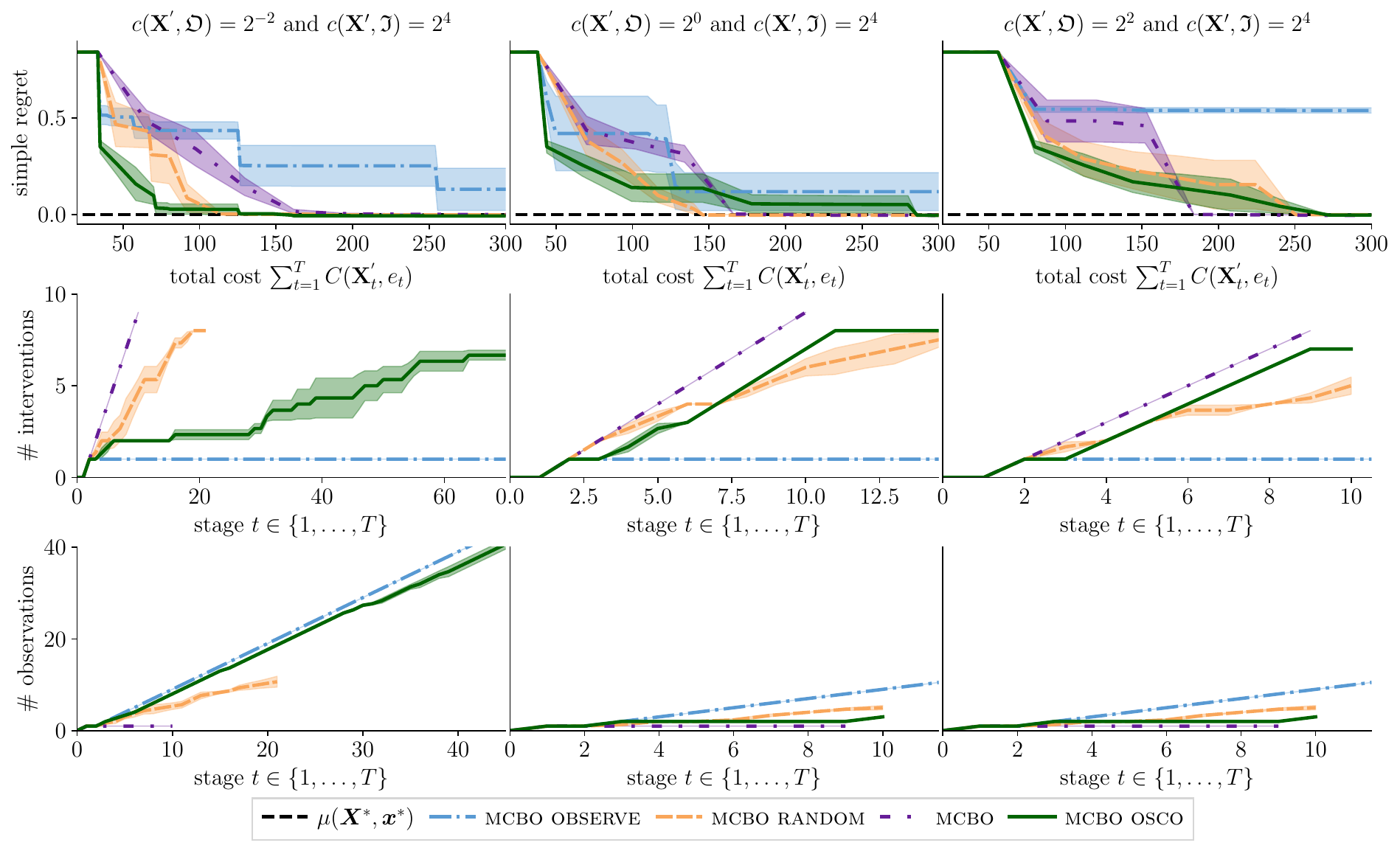}
}
    \caption{Convergence curves for \mcbo with different policies for balancing the intervention-observation trade-off when solving \labelcref{eq:cbo} for the PSA \scm (\cref{fig:psa_level}); the columns relate to evaluations with different observation costs $c(\X^{'}, \mathfrak{O})$; the top row shows the objective value against the evaluation cost; the middle and bottom rows show the number of interventions ($e_t=\mathfrak{I}$) and observations ($e_t=\mathfrak{O}$) at different stages of the optimisation; the curves indicate the mean and standard deviation ($\pm \frac{\sigma}{\sqrt[]{3}}$) over three evaluations with different random seeds.}
     \label{fig:psa_mcbo_convergence}
\end{figure}




\subsection{Synthetic \scm}\label{sec:extra_synthetic}
\Cref{fig:synthetic_convergence} shows convergence curves of \cbo with different policies for balancing the intervention-observation trade-off for the synthetic \scm with the causal graph in \cref{fig:graph_two}. We observe that both the policy that always observes and \cbo with \osco performs best on average. This result suggests to us that the most cost-effective way to find the optimal intervention for this \scm is to collect observations and estimate the causal effects via the do-calculus.

\begin{figure}
  \centering
  \scalebox{1}{
    \includegraphics[width=1\linewidth]{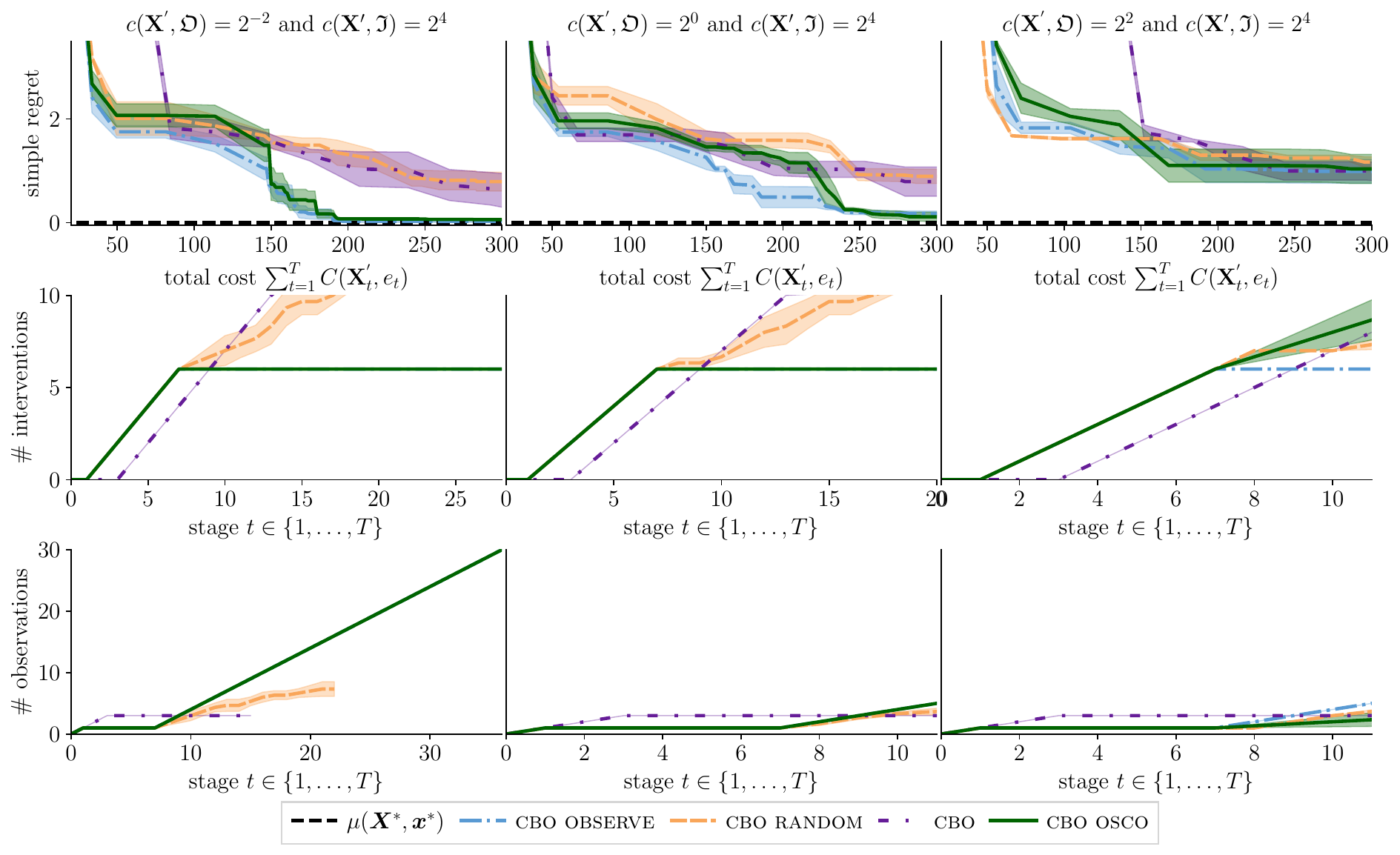}
}
    \caption{Convergence curves for \cbo with different policies for balancing the intervention-observation trade-off when solving \labelcref{eq:cbo} for the synthetic \scm in \cref{fig:graph_two}; the columns relate to evaluations with different observation costs $c(\X^{'}, \mathfrak{O})$; the top row shows the objective value against the evaluation cost; the middle and bottom rows show the number of interventions ($e_t=\mathfrak{I}$) and observations ($e_t=\mathfrak{O}$) at different stages of the optimisation; the curves indicate the mean and standard deviation ($\pm \frac{\sigma}{\sqrt[]{3}}$) over three evaluations with different random seeds.}
     \label{fig:synthetic_convergence}
\end{figure}

\clearpage
\section{Hyperparameters}\label{sec:hyperparameters}
Herein we provide the hyperparameters used for all the experiments. See \cref{fig:hyperparam_dag} for the relevant \DAG{}s and their corresponding hyperparameter tables.
\begin{figure}[ht!]
    \centering
    \begin{subfigure}[t]{0.35\textwidth}
        \centering
        \resizebox{\textwidth}{!}{%
           \begin{tikzpicture}[node distance =1.5cm]
                \node[state,circle] (X0) {$X$};
                \node[state,circle, right of = X0] (Z0) {$Z$};
                \node[state,circle, right of = Z0] (Y0) {$Y$};
                \path (X0) edge (Z0);
                \path (Z0) edge (Y0);
            \end{tikzpicture}
        }%
        \caption{\Cref{tab:toygraph_hparams}.}
    \end{subfigure}
    \hspace{2em}
    \begin{subfigure}[t]{0.35\textwidth}
        \centering
        \resizebox{\textwidth}{!}{%
           \begin{tikzpicture}[node distance =1.5cm]
                \node[state,circle] (X0) {$X$};
                \node[state,circle, right of = X0] (Z0) {$Z$};
                \node[state,circle, right of = Z0] (Y0) {$Y$};
                \path (X0) edge (Z0);
                \path (Z0) edge (Y0);
                \path[bidirected] (X0) edge[bend left=50] (Y0);
            \end{tikzpicture}
        }%
        \caption{\Cref{tab:toygraph_hparams} and \cref{sec:chain_scm_w_uc}.}
    \end{subfigure}
    \vspace{1em}
    \\
    \begin{subfigure}[t]{0.35\textwidth}
        \centering
        \resizebox{\textwidth}{!}{%
           \begin{tikzpicture}[node distance =1.5cm]
                \node[state,circle] (S) {$S$};
                \node[state,circle, below right of = S] (W) {$W$};
                \node[state,circle,circle, below right of = W] (Y) {$Y$};
                \node[state,circle, above right of = Y] (X) {$X$};
                \node[state,circle, above right of = X] (Z) {$Z$};
                \node[state,above right of = W] (B) {$B$};
                \path (S) edge (B);
                \path (B) edge (W);
                \path (W) edge (Y);
                \path (B) edge (X);
                \path (X) edge (Y);
                \path (Z) edge (X);
                \path[bidirected] (Y) edge[bend right=45] (Z);
                \path[bidirected] (S) edge[bend right=45] (Y);
            \end{tikzpicture}
        }%
        \caption{\Cref{tab:graph_two_hparams}.}
    \end{subfigure}
    \hspace{2em}
    \begin{subfigure}[t]{0.35\textwidth}
        \centering
        \resizebox{\textwidth}{!}{%
           \begin{tikzpicture}[node distance =2cm]
                \node[state,fill=gray!50,circle] (A) {Age};                 
                \node[state,fill=gray!50,circle, right of = A] (B) {BMI};   
                \node[state,circle, right of = B] (C) {Asp.};                
                \node[state,circle, below of = A] (D) {Stat.};                
                \node[state,fill=gray!50,circle, right of = D] (E) {Can.};   
                \node[state,circle, right of = E] (F) {PSA};               

                \path (A) edge (B);
                \path (A) edge[bend left=30] (C);
                \path (A) edge (F);
                \path (A) edge (E);
                \path (A) edge (D);
                \path (B) edge (C);
                \path (B) edge (F);
                \path (B) edge (E);
                \path (B) edge (D);
                \path (C) edge (E);
                \path (C) edge (F);
                \path (D) edge (E);
                \path (D) edge[bend right=30] (F);
                \path (E) edge (F);
            \end{tikzpicture}
        }%
        \caption{\Cref{tab:psa_level_hparams}.}
    \end{subfigure}
    \caption{\DAG{}s used for the experimental section of this paper. Each \DAG is associated with a table of experimental hyperparameters, found in the caption of each figure. The causal \mab experiment uses the \DAG{} in sub-figure (c) with a different \scm.}
    \label{fig:hyperparam_dag}
\end{figure}
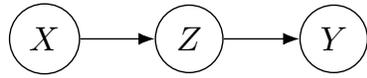
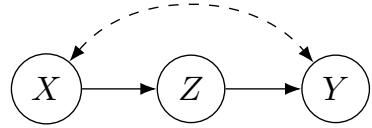
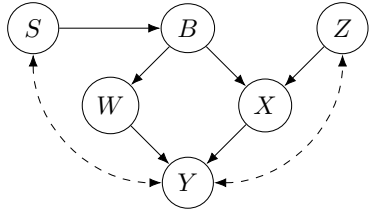
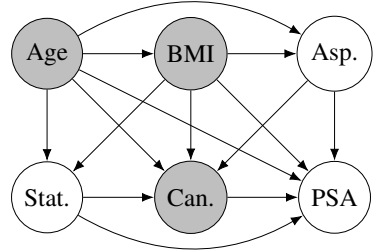
\FloatBarrier

\raggedbottom

\subsection{Hyperparameters for the chain \scm}

\begin{table}[H]
  \centering
  \caption{Hyperparameters for the chain \scm in \cref{fig:toygraph}.}\label{tab:toygraph_hparams}
  \begin{tabular}{lll}
  \toprule
    {\textit{Parameter}} & {\textit{Description}} & {\textit{Value}} \\ \midrule
    $\U$ & set of exogenous variables & $\{\epsilon_X,\epsilon_Z,\epsilon_Y\}$ \\
    $\V$ & set of endogenous variables & $\{X,Y,Z\}$ \\
    $\X$ & set of manipulative variables & $\{X,Z\}$ \\
    $\Y$ & set of target variables & $\{Y\}$ \\
    $\N$ & set of non-manipulative variables & $\{Y\}$ \\
    $\mat{F}$ & set of functions in the \scm & $\{f_X,f_Z,f_Y\}$ \\
    $\dom{X}$ & domain of the random variable $X$ & $[-5,5] \subset \mathbb{R}$ \\
    $\dom{Z}$ & domain of the random variable $Z$ & $[-5,20] \subset \mathbb{R}$ \\
    $\dom{Y}$ & domain of the random variable $Y$ & $[-5,5] \subset \mathbb{R}$ \\
    $f_X$ & function in the \scm & $f_X \colon X = \epsilon_X$ \\
    $f_Z$ & function in the \scm & $f_Z \colon Z = e^{-X} + \epsilon_Z$ \\
    $f_Y$ & function in the \scm & $f_Y \colon Y = \cos(Z) - e^{-\frac{Z}{20}} + \epsilon_y$ \\
    $\epsilon_X$ & Gaussian noise term in $f_X$ & $\epsilon_X \sim \mathcal{N}(\mu=0,\sigma=1)$ \\
    $\epsilon_Z$ & Gaussian noise term in $f_Z$ & $\epsilon_Z \sim \mathcal{N}(\mu=0,\sigma=0.5)$ \\
    $\epsilon_Y$ & Gaussian noise term in $f_Y$ & $\epsilon_Y \sim \mathcal{N}(\mu=0,\sigma=0.1)$ \\
    $K$ & evaluation budget in \labelcref{eq:formal_problem} & $300$ \\
    $\D_1$ & initial dataset of observations and interventions & $\emptyset$ \\
    $\widehat{\mu}$ & probabilistic model of $\mu$, see \labelcref{eq:cbo} & Gaussian process (\gp) \\
    $\widehat{\mat{F}}$ & probabilistic model of $\mat{F}$ & \gp \\
    \gp kernel & kernel for $\widehat{\mu}$ and $\widehat{\mat{F}}$ & causal \textsc{rbf} kernel \cite{cbo} \\
    \rbf length-scale & parameter for the causal \rbf kernel & $1$ \\
    \gp variance & parameter for the \gp in \citep{gpyopt2016} & $e^{-5}$ \\
    acquisition function & acquisition function for \cbo & \textsc{cei} \citep{cbo} \\
    $\eta$ & weighting factor in \labelcref{eq:stopping_reward} & $2$ \\
    $\kappa$ & weighting factor in \labelcref{eq:stopping_reward} & $1$ \\
    $\tau$ & weighting factor in \labelcref{eq:stopping_reward} & $5$ \\
    $\mathbf{M}_{\graph,Y}^{\mathbf{V}}$ & set of \mis{}s (\cref{def:mis})  & $\{\{Z\}, \{X,Z\}\}$ \\
    $\mathbf{P}_{\graph,Y}^{\mathbf{V}}$ & set of \pomis{}s (\cref{def:pomis}) & $\{\{Z\}\}$ \\
    $\mathbf{O}^{X}_{\graph,Y}$ & \mos (\cref{def:mos}) for $\myP{Y \mid \DO{X}{x}}$ & $\{X,Z,Y\}$ \\
    $\mathbf{O}^{Z}_{\graph,Y}$ & \mos (\cref{def:mos}) for $\myP{Y \mid \DO{Z}{z}}$ & $\{Z,Y\}$ \\
    $c(\cdot, \observe)$ & observation costs & $c(\mathbf{L}, \observe)=|\mathbf{L}|2^{-2} \quad \forall \mathbf{L} \in \mathcal{P}(\V)$\\
    $c(\cdot, \intervene)$ & intervention costs & $c(\mathbf{L}, \intervene)=|\mathbf{L}|2^{4} \quad\quad \forall \mathbf{L} \in \mathcal{P}(\X)$\\
    $\gamma$ & discount factor for the stopping problem & $1$ \\
    \mcbo batch size & batch size for $\pi_{\mathrm{O}}$ in \mcbo & $32$ \\
    \mcbo $\beta$ & exploration-exploitation parameter in \mcbo & $0.5$ \\
    \bottomrule
  \end{tabular}
\end{table}

\subsection{Hyperparameters for the chain \scm with an unobserved confounder}
\label{sec:chain_scm_w_uc}
The \scm in \cref{fig:admg_a} uses the same hyperparameters as \cref{fig:toygraph} listed in \cref{tab:toygraph_hparams}, with the differences that a) there is an unobserved confounder $\epsilon_{XY} \sim \mathcal{N}(\mu=0,\sigma=1)$ in both $f_Y$ and $f_X$; b) the set of \pomis{}s is $\mathbf{P}_{\graph,Y}^{\mathbf{V}} = \set{\emptyset,\set{Z}}$; c) the \mos for $\myP{Y \mid \DO{Z}{z}}$ is $\mathbf{O}^{Z}_{\graph,Y}=\{X,Z,Y\}$; and d) $\epsilon_Z \sim \mathcal{N}(\mu=-1,\sigma=0.5)$.

\subsection{Hyperparameters for the synthetic example \scm}
The functions $(\mathbf{F})$ for the synthetic example \scm in \cref{fig:graph_two}, adapted from \citep{cbo}, are given by
\begin{equation}
    \begin{aligned}[c]
        U_{SY}\  &= \epsilon_{SY}  &&\text{(unobserved confounder between $S$ and $Y$)} \\
        U_{ZY}\  &= \epsilon_{ZY}  &&\text{(unobserved confounder between $Z$ and $Y$)} \\
        f_S\colon S &= U_{SY} + \epsilon_S  && \\
        f_B \colon B &=  S + \epsilon_B &&\\
        f_Z\colon Z &= e^{-U_{ZY}} + \epsilon_Z  &&\\
        f_W\colon W &= \frac{e^{-B}}{10} + \epsilon_W &&\\
        f_X\colon X  &= \cos(Z) + \frac{B}{10} + \epsilon_X  &&\\
        f_Y\colon Y   &= \cos(W) + \sin(X) + U_{SY} + U_{ZY} \cdot \epsilon_Y &&
    \end{aligned}
    \label{eq:synthetic}
\end{equation}

Set of \mis{}s
\begin{align}
    \label{eq:synthetic_mis}
    &\set{\emptyset,\nonumber\\
    &\set{Z},\set{X}, \set{W}, \set{B}, \set{S}, \nonumber \\
    &\set{S, W}, \set{X, W}, \set{Z, W}, \set{B, Z}, \set{B, X}, \set{X, S},\set{B, W}, \set{S, Z},\nonumber \\
    &\set{Z, S,W}, \set{B, Z, W}}.
\end{align}
Set of \pomis{}s
\begin{align}
    \label{eq:synthetic_pomis}
    &\set{\emptyset,\nonumber\\
    &\set{X},\set{W},\set{Z},\nonumber\\
    &\set{B,W},\set{X,W},\set{Z,W}}.
\end{align}

\newpage
\begin{table}[H]
  \centering
  \caption{Hyperparameters for the synthetic \scm example in \cref{fig:graph_two}.}\label{tab:graph_two_hparams}
  \begin{tabular}{lll}
  \toprule
    {\textit{Parameter}} & {\textit{Description}} & {\textit{Value}} \\ \midrule
    $\U$ & set of exogenous variables & $\{\epsilon_{SY},\epsilon_{ZY},\epsilon,\set{\epsilon_i}_{i=S,B,Z,W,X}\}$ \\
    $\V$ & set of endogenous variables & $\{S,B,Z,W,X,Y\}$ \\
    $\X$ & set of manipulative variables & $\{S,B,Z,W,X\}$ \\
    $\Y$ & set of target variables & $\{Y\}$ \\
    $\N$ & set of non-manipulative variables & $\emptyset$ \\
    $\mat{F}$ & set of functions in the \scm & $\set{f_j}_{j =S,B,Z,W,X,Y}$ \\
    $\dom{S}$ & domain of the random variable $S$ & $[-5,4] \subset \mathbb{R}$\\
    $\dom{B}$ & domain of the random variable $B$ & $[-5,4] \subset \mathbb{R}$\\
    $\dom{W}$ & domain of the random variable $W$ & $[-5,5] \subset \mathbb{R}$\\
    $\dom{X}$ & domain of the random variable $X$ & $[-6,3] \subset \mathbb{R}$  \\
    $\dom{Z}$ & domain of the random variable $Z$ & $[-5,4] \subset \mathbb{R}$\\
    $\set{f_j}_{j \in \V }$ & functions in the \scm & see \labelcref{eq:synthetic} \\
    $\set{\epsilon_j}_{j \in \V \setminus \{W,X\}}$ & Gaussian noise terms & $\epsilon_j \sim \mathcal{N}(\mu=0,\sigma=0.1)$ \\
    $\epsilon_{W}$ & Gaussian noise term & $\epsilon_W \sim \mathcal{N}(\mu=0,\sigma=2)$ \\
    $\epsilon_{X}$ & Gaussian noise term & $\epsilon_W \sim \mathcal{N}(\mu=0,\sigma=2)$ \\
    $\epsilon_{SY}$ & Unobserved confounder term & $\epsilon_{SY} \sim \mathcal{N}(\mu=0,\sigma=0.1)$ \\
    $\epsilon_{ZY}$ & Unobserved confounder term & $\epsilon_{ZY} \sim \mathcal{N}(\mu=0,\sigma=0.1)$ \\
    $K$ & evaluation budget in \labelcref{eq:formal_problem} & $300$ \\
    $\D_1$ & initial dataset of observations and interventions & $\emptyset$ \\
    $\widehat{\mu}$ & probabilistic model of $\mu$, see \labelcref{eq:cbo} & Gaussian process (\gp) \\
    $\widehat{\mat{F}}$ & probabilistic model of $\mat{F}$ & \gp \\
    \gp kernel & kernel for $\widehat{\mu}$ and $\widehat{\mat{F}}$ & causal \textsc{rbf} kernel \cite{cbo} \\
    \rbf length-scale & parameter for the causal \rbf kernel & $1$ \\
    \gp variance & parameter for the \gp in \citep{gpyopt2016} & $e^{-5}$ \\
    acquisition function & acquisition function for \cbo & \textsc{cei} \citep{cbo} \\
    $\eta$ & weighting factor in \labelcref{eq:stopping_reward} & $2$ \\
    $\kappa$ & weighting factor in \labelcref{eq:stopping_reward} & $1$ \\
    $\tau$ & weighting factor in \labelcref{eq:stopping_reward} & $5$ \\
    $\mathbf{M}_{\graph,Y}^{\mathbf{V}}$ & \mis{}s (\cref{def:mis})  & see \labelcref{eq:synthetic_mis} \\
    $\mathbf{P}_{\graph,Y}^{\mathbf{V}}$ & \pomis{}s (\cref{def:pomis}) & see \labelcref{eq:synthetic_pomis}  \\
    $\mathbf{O}^{\emptyset}_{\graph,Y}$ & \mos (\cref{def:mos}) & $\V$ \\
    $\mathbf{O}^{X}_{\graph,Y}$ & \mos (\cref{def:mos}) & $\set{X,B,Z,Y}$ \\
    $\mathbf{O}^{W}_{\graph,Y}$ & \mos  & $\set{B,W,Y}$ \\
    $\mathbf{O}^{Z}_{\graph,Y}$ & \mos  & $\V$ \\
    $\mathbf{O}^{\set{B,W}}_{\graph,Y}$ & \mos & $\set{B,W,S,Y}$ \\
    $\mathbf{O}^{\set{X,W}}_{\graph,Y}$ & \mos & $\set{W,X,B,Z,Y}$ \\
    $\mathbf{O}^{\set{Z,W}}_{\graph,Y}$ & \mos & $\V$ \\
    $c(\cdot, \observe)$ & observation costs & $c(\mathbf{L}, \observe)=|\mathbf{L}|2^{-2} \quad \forall \mathbf{L} \in \mathcal{P}(\V)$\\
    $c(\cdot, \intervene)$ & intervention costs & $c(\mathbf{L}, \intervene)=|\mathbf{L}|2^{4} \quad\quad \forall \mathbf{L} \in \mathcal{P}(\X)$\\
    $\gamma$ & discount factor for the stopping problem & $1$ \\
    \bottomrule
  \end{tabular}
\end{table}

\newpage
\subsection{Hyperparameters for the PSA \scm}
The \DAG in \cref{fig:psa_level} describes the causal relationships between statin (node $D$), aspirin (node $C$) and prostate-specific antigen (PSA) level (node $F$), mediated by a set of non-manipulative variables, adapted from \citep{ferro2015use}. We use the same \scm as \citet[Appendix \S5]{cbo}.
\begin{equation}
    \begin{aligned}[c]
        f_A\ (\text{Age})\colon A &= \mathcal{U}(55,75) \\
        f_B\ (\text{BMI}) \colon B &= \mathcal{N}(27.0 - 0.01 \cdot A,0.7) \\
        f_C\ (\text{Aspirin})\colon C &= \sigma (-8.0 + 0.10 \cdot A + 0.03 \cdot B) \\
        f_D\ (\text{Statin}) \colon D &= \sigma (-13.0 + 0.10 \cdot A + 0.20 \cdot B) \\
        f_E\ (\text{Cancer}) \colon E &= \sigma (2.2 - 0.05\cdot A + 0.01 \cdot B - 0.04 \cdot D + 0.02 \cdot C) \\
        f_F\ (\text{PSA}) \colon F &= \mathcal{N}(6.8 + 0.04 \cdot A - 0.15 \cdot B - 0.60 \cdot D + 0.55 \cdot C +  E,0.4)
    \end{aligned}
    \label{eq:psa_sem}
\end{equation}
In \labelcref{eq:psa_sem}, $\mathcal{U}(a,b)$ denotes the continuous uniform distribution on the interval $[a,b]$, $\mathcal{N}(\mu,\sigma^2)$ denotes the univariate Gaussian distribution with mean $\mu$ and variance $\sigma^2$ and $\sigma(x)$ denotes the sigmoid function $\frac{1}{1+e^{-x}}$.
\begin{table}[H]
  \centering
  \begin{threeparttable}
  \caption{Hyperparameters for the \textsc{PSA} level example in \cref{fig:psa_level} .}\label{tab:psa_level_hparams}
  \begin{tabular}{lll}
  \toprule
    {\textit{Parameter}} & {\textit{Description}} & {\textit{Value}} \\ \midrule
    $\U$ & set of exogenous variables & $\emptyset$  \\
    $\V$ & set of endogenous variables & $\{A,B,C,D,E,F\}$ \\
    $\X$ & set of manipulative variables & $\{D,C\}$ \\
    $\N$ & set of non-manipulative variables & $\{A,B,E\}$ \\
    $\Y$ & set of target variables & $\{F\}$ \\
    $\mat{F}$ & set of functions in the \scm & $\{f_A,f_B,f_C,f_D,f_E,f_F\}$ \\
    $\dom{A}$ & domain of the random variable $A$\tnote{$\ddagger$} & $[55,75] \subset \mathbb{R}$ \\
    $\dom{B}$ & domain of the random variable $B$\tnote{$\dagger$} & $[24.1,28.8] \subset \mathbb{R}$  \\
    $\dom{C}$ & domain of the random variable $C$\tnote{$\ddagger$} &  $[0,1] \subset \mathbb{R}$\\
    $\dom{D}$ & domain of the random variable $D$\tnote{$\ddagger$} &  $[0,1] \subset \mathbb{R}$\\
    $\dom{E}$ & domain of the random variable $E$\tnote{$\ddagger$} &  $[0,1] \subset \mathbb{R}$\\
    $\dom{F}$ & domain of the random variable $F$\tnote{$\dagger$}& $[4.36,7.81] \subset \mathbb{R}$  \\
    $\set{f_i}_{i \in \V }$ & function in the \scm & see \labelcref{eq:psa_sem} \\
    $K$ & evaluation budget in \labelcref{eq:formal_problem} & $300$ \\
    $\D_1$ & initial dataset of observations and interventions & $\emptyset$ \\
    $\widehat{\mu}$ & probabilistic model for $\mu$, see \labelcref{eq:cbo} & \gp \\
    $\widehat{\mat{F}}$ & probabilistic model for $\mat{F}$ &\gp \\
    \gp kernel & kernel for $\widehat{\mu}$ and $\widehat{\mat{F}}$ & causal \textsc{rbf} kernel \cite{cbo} \\
    \rbf length-scale & parameter for the causal \rbf kernel & $1$ \\
    \gp variance & parameter for the \gp in \citep{gpyopt2016} & $e^{-5}$ \\
    acquisition function & acquisition function for \cbo & \textsc{cei} \citep{cbo} \\
    $\eta$ & weighting factor in \labelcref{eq:stopping_reward} & $2$ \\
    $\kappa$ & weighting factor in \labelcref{eq:stopping_reward} & $1$ \\
    $\tau$ & weighting factor in \labelcref{eq:stopping_reward} & $5$ \\
    $\mathbf{M}_{\graph,F}^{\V}$ & set of \mis{}s (\cref{def:mis})  & $\{\emptyset, \{C\},\{D\},\{C,D\} \}$ \\
    $\mathbf{P}_{\graph,F}^{\V}$ & set of \pomis{}s (\cref{def:pomis}) & $\{\{C,D\} \}$ \\
    $\mathbf{O}^{\{C,D\}}_{\graph,Y}$ & \mos (\cref{def:mos}) for $\myP{F \mid \DO{C}{c},\DO{D}{d}}$ & $\{A,B,C,D,F\}$ \\
    $c(\cdot, \observe)$ & observation costs & $c(\mathbf{L}, \observe)=|\mathbf{L}|2^{-2} \quad \forall \mathbf{L} \in \mathcal{P}(\V)$\\
    $c(\cdot, \intervene)$ & intervention costs & $c(\mathbf{L}, \intervene)=|\mathbf{L}|2^{4} \quad\quad \forall \mathbf{L} \in \mathcal{P}(\X)$\\
    $\gamma$ & discount factor for the stopping problem & $1$ \\
    \mcbo batch size & batch size for $\pi_{\mathrm{O}}$ in \mcbo & $32$ \\
    \mcbo $\beta$ & exploration-exploitation parameter in \mcbo & $0.5$ \\
    \bottomrule
  \end{tabular}
    \begin{tablenotes}
      \small
      \item[$\dagger$]The domains used in this experiment are the 25th and 75th percentiles of the measured variables, found in \citep[Table 1]{ferro2015use}.
      \item[$\ddagger$]Strictly this is a discrete variable which we have made continuous for computational reasons.
    \end{tablenotes}
    \end{threeparttable}
\end{table}

Where
\begin{equation}
    \myP{F \mid \DO{C}{c},\DO{D}{d}}  = \int \myP{F \mid A,B,C,D} \myP{A,B} \mathrm{d}a\mathrm{d}b.
\end{equation}

\subsection{Synthetic causal \mab}
We consider the causal \mab setting in which the causal structure is provided by the \scm with \DAG{} $\graph$ in \cref{fig:graph_two}. The two unobserved confounders follow binary distributions governed by
\begin{align*}
    P(U_{SY} = 1) &= 0.1\\
    P(U_{ZY} = 1) &= 0.05
\end{align*}
and exogenous variables follow binary distributions governed by
\begin{align*}
   P(U_S = 1) &= 0.45\\
    P(U_B = 1) &= 0.4\\
    P(U_Z = 1) &= 0.8\\
    P(U_W = 1) &= 0.3\\
    P(U_X = 1) &= 0.85\\
\end{align*}
All variables are binary with domain $\{0,1\}$ and with $\mat{F}$:
\begin{equation}
    \begin{aligned}[c]
        f_S\colon S &= U_{SY} \oplus U_S \\
        f_Z\colon Z &= 1 - U_{ZY} \oplus U_Z \\
        f_B\colon B &= S \oplus U_B  \\
        f_W\colon W &= B \oplus  U_W \\
        f_X\colon X &= 1 - B \oplus Z \oplus U_X \\
        f_Y\colon Y &= W\oplus X \oplus U_{SY} \oplus U_{ZY} \\
    \end{aligned}
\end{equation}
where $\oplus$ is the exclusive-or function. The set of \pomis{}s, the set of \mis{}s and the \mos for each intervention can be found in \cref{tab:graph_two_hparams}.

\section{Pseudocode and implementation of \osco}\label{sec:algorithms}
The optimal stopping problem described \cref{sec:os_formulation} can be integrated with existing causal optimisation algorithms to balance the \textit{intervention}-\textit{observation} trade-off. More specifically, given an optimisation policy $\pi_{\mathrm{O}}$ that determines which intervention to evaluate at each stage of the optimisation, the solution to the optimal stopping problem in \labelcref{eq:stopping_time_problem} determines whether the intervention should be evaluated by intervention or observation. ($\pi_{\mathrm{O}}$ may for example be implemented by the \cbo algorithm \cite[Alg. 1]{cbo} or the causal \mab algorithm in \cite[Alg.1 ]{causal_bandits_3}) The pseudocode for integrating the optimal stopping problem with the existing algorithms is listed in \cref{alg:os_pseudocode}. The main computational complexity of the integration is the repeated solving of \labelcref{eq:olsa}, which requires evaluating a potentially high-dimensional integral and evaluating the optimisation policy $\pi_{\mathrm{O}}$ several times. The integral can be evaluated efficiently using Monte-Carlo methods \cite{rubinstein_mc,evans2000approximating} and the evaluations of $\pi_{\mathrm{O}}$ (which may involve optimisation of an acquisition function as is e.g. the case in \cbo \cite{cbo}) can be done in parallel. The average execution times per iteration when running \cbo and \mcbo with and without \osco are shown in \cref{fig:runtimes}.
\begin{algorithm}[H]
  \caption{Optimal stopping for Causal Optimisation (\osco).}\label{alg:os_pseudocode}
  \hspace*{\algorithmicindent}\textbf{Input:} Optimisation policy $\pi_{\mathrm{O}}$ and causal graph $\mathcal{G}$.\\
  \hspace*{\algorithmicindent}\textbf{Output:} Optimised intervention $(\X^{\star}, \x^{\star})$.
\begin{algorithmic}[1]
  \Procedure{}{}
  \State Set $\D_1 \triangleq \emptyset$ and initialise models $\widehat{\mu}_{\D_1},\widehat{\mat{F}}_{\D_1}$.
  \State Compute set of \pomis{s} $\mathbf{P}_{\graph,Y}^{\mathbf{V}}$.
  \For{$t = 1,\hdots,T$}
  \State Set $\X_t^{'}, \x^{'}_t \triangleq \pi_{\mathrm{O}}(\widehat{\mu}_{\D_t},\widehat{\mat{F}}_{\D_t},\mathbf{P}_{\graph,Y}^{\mathbf{V}})$.
  \State Compute $\mathcal{T}^{*}$ using \labelcref{eq:olsa}.
  \If{$\mathcal{T}^{*} = 1$}
  \State Intervene $\DO{\X^{'}_t}{\x^{'}_t}$ \& measure $Y$.
  \State Set $\D_{t+1}\triangleq \D_t \cup \{(\DO{\X^{'}_t}{\x^{'}_t},Y)\}$.
  \Else
  \State Observe $\mathbf{o}_t \sim P(\mathbf{O}^{\mathbf{X}}_{\graph,Y})$.
  \State Set $\D_{t+1}\triangleq \D_t \cup \{\mathbf{o}_t\}$.
  \EndIf
  \State Update models $\widehat{\mu}_{\D_{t+1}}$ and $\widehat{\mat{F}}_{\D_{t+1}}$.
  \EndFor
  \State \Return
  \begin{align*}
(\X^{\star}, \x^{\star}) \in     \argmin_{
        \substack{\X' \in \mathbf{P}_{\graph,Y}^{\mathbf{V}}; \\ \x' \in \dom{\X'}}
    }\widehat{\mu}(\X',\x')
  \end{align*}
\EndProcedure
\end{algorithmic}
\end{algorithm}

\begin{figure}[!htb]
    \begin{subfigure}[t]{0.6\textwidth}
        \centering
        \resizebox{1\textwidth}{!}{%
            \includegraphics{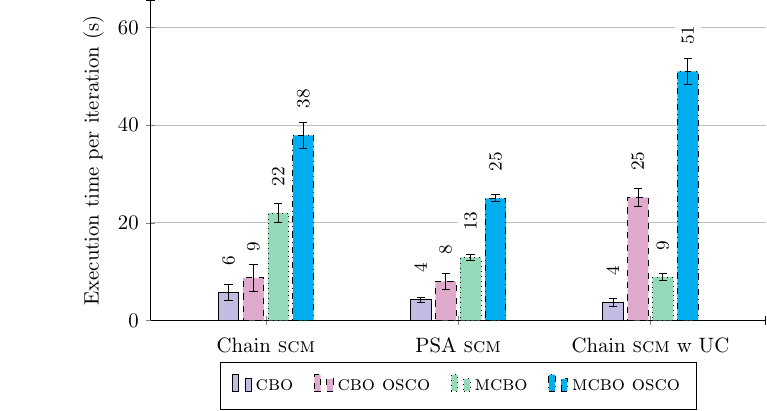}
        }
        \caption{Causal \bo algorithms.}
    \end{subfigure}
    \hfill
    \begin{subfigure}[t]{0.4\textwidth}
    \centering
        \resizebox{0.58\textwidth}{!}{%
            \includegraphics{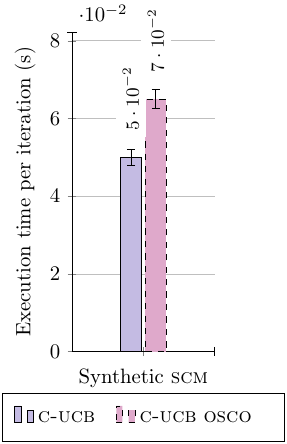}
        }
    \caption{Causal \mab algorithms.}
    \end{subfigure}
    \caption{Execution times per iteration when running \cbo \cite{cbo}, \mcbo \cite{model_based_cbo} and \textsc{c-ucb} \cite{causal_bandits_3} with and without \osco for the example \scm{}s defined in \cref{sec:experiments}; the height of each bar indicates the mean execution time from $100$ measurements and the error bars indicate the standard deviations.}
 \label{fig:runtimes}
\end{figure}

\section{Background on Markovian optimal stopping problems}\label{sec:mdp_background}
This appendix provides a self-contained background on \mdp{}s and Markovian optimal stopping problems. It provides sufficient prerequisite knowledge for a reader that is not familiar with optimal stopping to follow \cref{sec:os_formulation} in the main body of the paper.
\subsection{Markov decision processes}
A Markov Decision Process (\mdp) models the control of a discrete-time dynamical system that evolves in time-steps from $t=1$ to $t=T$ and is defined by the seven-tuple \cite{bellman1957markovian,puterman}:
\begin{align}
\mathcal{M} = \langle \mat{S}, \mat{A}, P_{\mathcal{M}}, r, \gamma, \rho_1, T \rangle\label{eq:mdp_def}
\end{align}
$\mat{S} \subseteq \mathbb{R}^n$ denotes the set of states, $\mat{A} \subseteq \mathbb{R}^{m}$ denotes the set of actions, $\gamma \in \left[0,1\right]$ is a discount factor, $\rho_1 : \mat{S} \rightarrow [0,1]$ is the initial state distribution and $T$ is the time horizon. $P_{\mathcal{M}}\left(\mathbf{S}_{t+1}=\mathbf{s}_{t+1}\mid \mathbf{S}_t=\mathbf{s}_t, \mathbf{A}_t=a_t\right)$ refers to the probability of transitioning from state $\mathbf{s}_t$ to state $\mathbf{s}_{t+1}$ when taking action $\mathbf{a}_t$ and satisfies the Markov property $P_{\mathcal{M}}\left(\mathbf{S}_{t+1}=\mathbf{s}_{t+1} \mid \mathbf{S}_t=\mathbf{s}_t\right) = P_{\mathcal{M}}\left(\mathbf{S}_{t+1}=\mathbf{s}_{t+1} \mid  \mathbf{S}_1=\mathbf{s}_1, \hdots, \mathbf{S}_t=\mathbf{s}_t\right)$, where $\mathbf{s}_t\in \mat{S}$ and $\mathbf{a}_t \in \mat{A}$ are realisations of the random vectors $\mathbf{S}_t$ and $\mathbf{A}_t$. Similarly, $r(\mathbf{s}_t,\mathbf{a}_t) \in \mathbb{R}$ is the reward when taking action $\mathbf{a}_t$ in state $\mathbf{s}_t$, which we assume is bounded, i.e. $|r(\mathbf{s}_t,\mathbf{a}_t)| \leq M < \infty$ for some $M \in \mathbb{R}$. If $P_{\mathcal{M}}$ and $r(\mathbf{s}_t,\mathbf{a}_t)$ are independent of the time-step $t$, the \mdp is said to be \textit{stationary} and if $\mat{S}$ and $\mat{A}$ are finite, the \mdp is said to be \textit{finite}.

A policy is a function $\pi: \{1,\hdots, T\} \times \mat{S} \rightarrow \Delta(\mat{A})$. If a policy is independent of the time-step $t$ given the current state, it is called \textit{stationary}. An optimal policy $\pi^{*}$ maximizes the expected discounted cumulative reward over the time horizon:
\begin{align}
\pi^{*} &\in \argmax_{\pi \in \mat{\Pi}} \mathbb{E}_{\pi}\left[\sum_{t=1}^{T}\gamma^{t-1}R_{t}\right] \label{eq:rl_prob}
\end{align}
where $\mat{\Pi}$ is the policy space, $R_{t} \in \mathbb{R}$ is a random variable representing the reward at time $t$ and $\mathbb{E}_{\pi}$ denotes the expectation of the random vectors and variables $(\mathbf{S}_t,R_t,\mathbf{A}_t)_{t=1,\hdots,T}$ under policy $\pi$.

Optimal deterministic policies exist for a finite \mdp with bounded rewards and either $T < \infty$ or $\gamma \in [0,1)$ \cite[Prop. 4.4.3 \& Thm. 6.2.10]{puterman}. If the \mdp is also stationary and the horizon is either random or infinite with $\gamma \in [0,1)$, an optimal stationary policy exists \cite[Thm. 6.2.10]{puterman}.

The Bellman equations relate any optimal policy $\pi^{*}$ to the two value functions $V^{*} : \mat{S} \rightarrow \mathbb{R}$ and $Q^{*}: \mat{S} \times \mat{A} \rightarrow \mathbb{R}$ \cite{bellman_eq}:
\begin{align}
V^{*}(\mathbf{s}_t) &= \displaystyle\max_{a_t\in \mat{A}} \mathbb{E}_{\mathbf{S}_{t+1}}\big[R_{t+1} + \gamma V^{*}(\mathbf{S}_{t+1}) \mid \mathbf{S}_t=\mathbf{s}_t, \mathbf{A}_t=\mathbf{a}_t\big]\label{eq:bellman_eq_31} \\
Q^{*}(\mathbf{s}_t,\mathbf{a}_t) &= \mathbb{E}_{\mathbf{S}_{t+1}}\big[R_{t+1} + \gamma V^{*}(\mathbf{S}_{t+1}) | \mathbf{S}_t=\mathbf{s}_t, \mathbf{A}_t=\mathbf{a}_t\big] \label{eq:bellman_eq_33}\\
\pi^{*}(\mathbf{s}_t) &\in \argmax_{\mathbf{a}_t\in \mat{A}} Q^*(\mathbf{s}_t,\mathbf{a}_t)\label{eq:bellman_eq_34}
\end{align}
where $V^{*}(\mathbf{s}_t)$ and $Q^{*}(\mathbf{s}_t,\mathbf{a}_t)$ denote the expected cumulative discounted reward under $\pi^{*}$ for each state and state-action pair, respectively. Solving \labelcref{eq:bellman_eq_31} -- \labelcref{eq:bellman_eq_33} means computing the value functions from which an optimal policy can be obtained via \labelcref{eq:bellman_eq_34}.

\subsection{Markovian optimal stopping problems}
Optimal stopping is a classical problem domain with a well-developed theory \cite{wald,shirayev,stopping_book_1,chow1971great,bert05,bather_decision_theory,puterman,hammar_stadler_tnsm}. Many variants of the optimal stopping problem have been studied. For example, discrete-time and continuous-time problems, stationary and non-stationary problems and Markovian and non-Markovian problems. As a consequence, different solution methods for these variants have been developed. The most commonly used methods are the \textit{martingale approach} \cite{stopping_book_1,chow1971great,Snell1952TAMS} and the \textit{Markovian approach} \cite{shirayev,bert05,puterman,ross_stochastic_dp,bather_decision_theory}.

In this paper, we focus on a stationary optimal stopping problem with a finite time horizon $T$, discrete-time progression, a continuous state space $\mathbf{S} \subset \mathbb{R}^{n}$, bounded rewards and the Markov property. We use the Markovian solution approach and model the problem as a stationary \mdp $\mathcal{M}$, where the system state evolves as a discrete-time Markov process $(\mathbf{S}_{t})_{t=1}^{T}$. Here $\mathbf{S}_{t} \in \mathbf{S}$ and $\mathbf{s}_t$ denotes the realization of $\mathbf{S}_{t}$. At each time-step $t$ of this process, two actions are available: ``stop'' ($\mathrm{S}$) and ``continue'' ($\mathrm{C}$) i.e. $(A_t\in \{ \mathrm{S,C} \})_{t=1,\ldots,T}$. The \textit{stop} action yields a reward $r(\mathbf{s}_t,\mathrm{S})$ and terminates the process. In contrast, the \textit{continue} action causes the process to transition to the next state according to the transition probabilities $P_{\mathcal{M}}$ and yields the reward $r(\mathbf{s}_t,\mathrm{C})$.

A \textit{stopping time} is a positive random variable $1 \leq \T \leq T$ that is dependent on $s_1,\hdots,s_{\mathcal{T}}$ and independent of $s_{\mathcal{T}+1},\hdots s_{T}$ \cite{stopping_book_1}:
\begin{align}
\mathcal{T} &= \inf\{t: t \geq 1, \text{ }a_t=\mathrm{S}\}, \label{eq:stopping_time_def_1}
\end{align}
The objective is to find a deterministic and stationary stopping policy $\pi^{*}: \mathbf{S} \rightarrow \{\mathrm{S},\mathrm{C}\}$ that maximizes the expected discounted cumulative reward of the induced stopping time $\mathcal{T}$:
\begin{align}
  &\pi^{*} \in \argmax_{\pi \in \Pi} \mathbb{E}_{\pi}\left\{\sum_{t=1}^{\mathcal{T}-1}\gamma^{t-1}r(\mathbf{S}_t,\mathrm{C})  + r(\mathbf{S}_{\mathcal{T}},\mathrm{S})\right\}\label{eq:optimal_stopping_2}
\end{align}
Due to the Markov property, any policy that satisfies \labelcref{eq:optimal_stopping_2} also satisfies the following Bellman equation:
\begin{align}
\pi^{*}(\mathbf{s}) \in \argmax_{\{\mathrm{S},\mathrm{C}\}} \left\{\underbrace{r(\mathbf{s},\mathrm{S})}_{\text{stop } (\mathrm{S})}, \text{ }\underbrace{\mathbb{E}_{S^{'}}\left[r(\mathbf{s},\mathrm{C}) + \gamma V^{*}(\mathbf{S}^{\prime})\right]}_{\text{continue } (C)}\right\}\label{eq:optimal_stopping_1}
\end{align}
where $V^{*}$ is defined in \labelcref{eq:bellman_eq_31}.

\end{appendices}


\end{document}